\documentclass[journal]{IEEEtran} % uncomment this

\usepackage{setspace}
\usepackage{psfrag,subfigure}
\usepackage{amssymb, amsthm, amsmath, cite}
\usepackage{yhmath}
\usepackage{color}
\usepackage{verbatim}
\usepackage{mathrsfs}
\usepackage{graphicx}
\usepackage{algorithm}
\usepackage{algpseudocode}
\usepackage{stmaryrd}

\usepackage[normalem]{ulem}
\usepackage[top=1in, bottom=0.7in, left=0.7in, right=0.7in]{geometry}
\usepackage{mathtools}
\usepackage{url}
\usepackage{pdfpages}
\usepackage{cite}
\usepackage{multicol}

\usepackage{enumitem}
\setlist[description]{leftmargin=0.3cm,labelindent=0cm}

\newtheorem{theorem}{Theorem}[section]

\newtheorem{lemma}{Lemma}[section]
\newtheorem{corollary}{Corollary}[section]
\newtheorem{defn}{Definition}[section]
\newtheorem{assum}{Assumption}[section]
\newtheorem{remark}{Remark}[section]
\newtheorem{problem}{Problem}[section]
\newtheorem{example}{Example}[section]
\IEEEoverridecommandlockouts % uncomment this
\newcommand{\RR}{{\mathbb R}}

\newcommand{\ZZ}{{\mathbb Z}}

\newcommand{\cC}{{\mathcal C}}

\newcommand{\cF}{{\mathcal F}}
\newcommand{\cG}{{\mathcal G}}

\newcommand{\cI}{{\mathcal I}}

\newcommand{\cM}{{\mathcal M}}

\newcommand{\cP}{{\mathcal P}}

\newcommand{\cX}{{\mathcal X}}

\newcommand{\cAOTS}{{\mathcal A}_{\textsc{\tiny OTS}}}
\newcommand{\cAPA}{{\mathcal A}_{\textsc{\tiny PA}}}
\newcommand{\OTS}{{\textsc{\tiny OTS}}}
\newcommand{\cHMA}{{\mathcal H}_{\textsc{\tiny MA}}}
\newcommand{\EMA}{E_{\textsc{\tiny MA}}}
\newcommand{\XMA}{X_{\textsc{\tiny MA}}}
\newcommand{\IMA}{I_{\textsc{\tiny MA}}}
\newcommand{\GMA}{G_{\textsc{\tiny MA}}}
\newcommand{\RMA}{R_{\textsc{\tiny MA}}}
\newcommand{\LOTS}{L_{\textsc{\tiny OTS}}}
\newcommand{\EOTS}{E_{\textsc{\tiny OTS}}}
\newcommand{\EPA}{E_{\textsc{\tiny PA}}}
\newcommand{\DPA}{D_{\textsc{\tiny PA}}}
\newcommand{\HPA}{H_{\textsc{\tiny PA}}}

\newcommand{\SMA}{\Sigma_{\textsc{\tiny MA}}}
\newcommand{\PiPA}{\Pi_{\textsc{\tiny PA}}}
\newcommand{\QPA}{Q_{\textsc{\tiny PA}}}
\newcommand{\QMA}{Q_{\textsc{\tiny MA}}}
\newcommand{\SPA}{\Sigma_{\textsc{\tiny PA}}}
\newcommand{\phiMA}{\phi_{\textsc{\tiny MA}}}
\newcommand{\yMA}{y_{\textsc{\tiny MA}}}

\newcommand{\oEMA}{\overline{E}_{\textsc{\tiny MA}}}
\newcommand{\oSMA}{\overline{\Sigma}_{\textsc{\tiny MA}}}

\newcommand{\sH}{\mathscr{H}}
\newcommand{\sF}{\mathscr{F}}
\newcommand{\sB}{\mathscr{B}}

\newcommand{\ol}{\overline}

\newcommand{\conv}{\textup{co}}

\newcommand\tqed{\leavevmode\unskip\penalty9999 \hbox{}\nobreak\hfill\quad\hbox{$\triangleleft$}}
\DeclareMathOperator*{\argmin}{\arg\!\min}

\title{\Large \bf
A Modular Framework for Motion Planning using Safe-by-Design Motion Primitives}

\author{Marijan Vukosavljev, Zachary Kroeze, Angela P. Schoellig, and Mireille E. Broucke% <-this % stops a space
\thanks{Marijan Vukosavljev, Zachary Kroeze, and Mireille E. Broucke are with the Dept. of Electrical and Computer Engineering, 
University of Toronto, Canada (e-mails: mario.vukosavljev@mail.utoronto.ca, zach.kroeze@mail.utoronto.ca, 
broucke@control.utoronto.ca). Angela P. Schoellig is with the University of Toronto Institute for Aerospace Studies (UTIAS), 
Canada (email: schoellig@utias.utoronto.ca). Supported by the Natural Sciences and Engineering Research Council of Canada (NSERC). }%
}

\begin{document}
%\linespread{0.9} % in case tight for space!
\maketitle
\begin{abstract}
We present a modular framework for solving a motion planning problem among a group of robots. The proposed framework utilizes a finite set of low level {\em motion primitives} to generate motions in a gridded workspace. The constraints 
on allowable sequences of motion primitives are formalized through a {\em maneuver automaton}. At the high level, a control policy determines 
which motion primitive is executed in each box of the gridded workspace. We state general conditions on 
motion primitives to obtain provably correct behavior so that a library of safe-by-design motion primitives 
can be designed. The overall 
framework yields a highly robust design by utilizing feedback strategies at both the low and high levels. We provide specific designs for motion primitives and control policies suitable for multi-robot motion planning; the modularity of our approach enables one to independently customize the designs of each of these components.
Our approach is experimentally validated on a group of quadrocopters.
\end{abstract}

\section{Introduction}
\label{sec:intro}

This paper presents a modular, hierarchical framework for motion planning and control of robotic systems. While motion planning has received a great deal of attention by many researchers, because the problem is highly complex especially when there are several robotic agents working together in a cluttered environment, significant challenges remain. 
%\textcolor{blue}{The problem of motion planning for multi-robot systems is a highly active area of research, encompassing a wide variety applications such as automated warehouse operation, transportation of goods, and surveillance \cite{}.
%However, due to the high complexity of the problem}, especially when there are many robotic agents working together in a cluttered environment, significant challenges remain. 
Hierarchy, in which the control design has several layers, is an architectural strategy to overcome this complexity. Almost all hierarchical frameworks for motion planning aim to balance flexibility in the control specification at the high level, guarantees on correctness and safety at the low level, and computational feasibility overall. 

Historically motion planning was focused on high level planning algorithms, while suppressing details on the dynamic capabilities of the robots at the low level \cite{LAV06}. 
Taking full account of low level dynamics in combination with solving the high level planning problem can lead to a computationally intractable problem. 
%Recent literature on this problem includes \cite{CHOS03,LAV09,PAP09,DIM13,GAZ14,KUM14,SCHW14}. 
Despite the wealth of available research \cite{CHOS03,KUM14,AYAN10,BELTA08}, computationally efficient solutions to the motion planning problem with tight integration of high and low levels are highly sought after. 

\begin{figure}%[t]
\centering%
\includegraphics[width=0.9\linewidth,trim=0cm 0cm 0cm 0cm, clip=true]{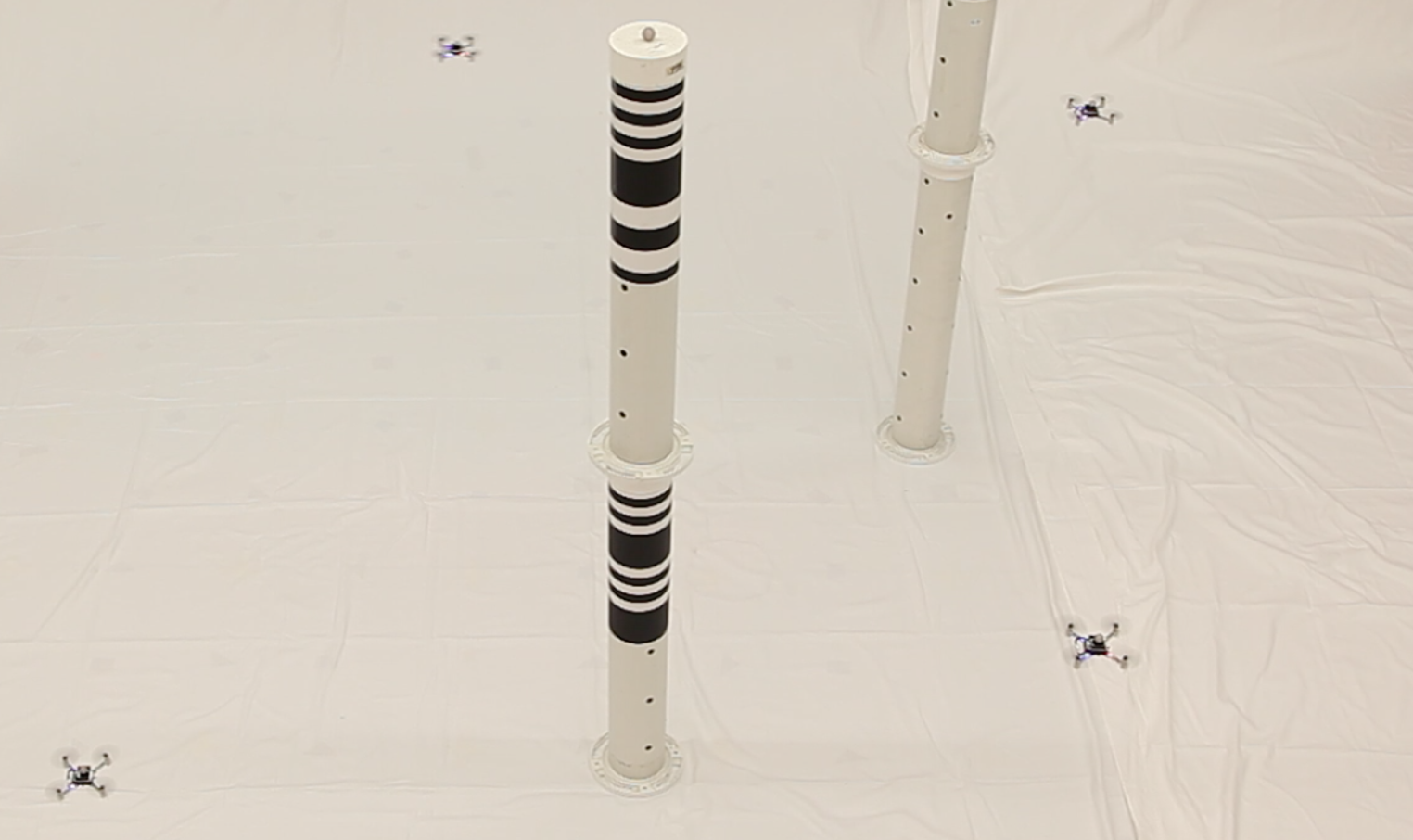}
\caption{Crazyflie quadrocopters navigate in a cluttered environment. Video results are available at http://tiny.cc/modular-3alg.}
\vspace{-2mm}
\label{fig:exp_setup}%
\end{figure}

We propose a modular hierarchical framework so that one can independently plug and play both low level controllers 
and high level planning algorithms in order to realize a balance between flexibility at the high level, 
safety at the low level, and computational feasibility.
To make a customizable approach feasible, we introduce three assumptions. First, the output space of the underlying dynamical system has translational symmetry, namely position invariance, a property satisfied by many robotic models \cite{FRAZ05}. Second, the output space is gridded uniformly into rectangular boxes. Finally, the control capabilities are discretized into a finite set of {\em motion primitives}, where the low level describes the implementation of the motion primitives while the high level selects the motion primitives. Together, these assumptions imply that motion primitives can be designed over a single box, so that they can then be reapplied to any other box.

Now we give an overview of the features and techniques we employ, and we highlight other frameworks that share those features. 
%\textcolor{blue}{This paper aims to provide a} 
We  provide general formulation of motion primitives for nonlinear systems so that they can be applied to multi-robot systems.
%\textcolor{blue} but potentially also to other robotic systems.
We focus on reach-avoid specifications in a priori known environments, in which the system must reach a desired configuration in a safe manner \cite{AYAN10,CHOS03,SATT14,LAV09}. Reach-avoid offers a fairly rich behavior set so that, for instance, a fragment of linear temporal logic (LTL) can 
be encoded as a sequence of reach-avoid problems \cite{WOL13}, as we also show in our applications.

As we have mentioned, we abstract the output space into rectangular regions \cite{KUM14} rather than more general polytopic regions \cite{CHOS03,PAP09,BELTA08,LAV09} in order to exploit symmetry.
% in order to be able to implement the same motion primitives over any box. 
Motion primitives have been employed in various ways \cite{KUM14, SATT14} and we encode feasible sequences of motion primitives by a {\em maneuver automaton} \cite{FRAZ05}.
In contrast to the motion primitive methods above, our implementation of the low-level control design of motion primitives is based on {\em reach control theory}\cite{RB06,BG14}, which provides a highly flexible and intuitive set of design tools that have two notable advantages over tracking: first, it is not necessary to find feasible open-loop trajectories to track; second, safety constraints on the system states during the execution and concatenation of motion primitives can be guaranteed by design.
Finally, planning at the high level is based on standard {\em shortest path algorithms} \cite{LAV06, BRO05} applied to the graph arising from the synchronous product of the discrete part of the maneuver automaton and the graph arising from the output space partition. The high-level plan generates a {\em control policy}, which selects the motion primitives over the gridded output space.
The modularity of our approach enables one to employ other closed-loop methods such as potential methods \cite{CHOS03} or vector-field shaping \cite{LAV09} for low-level control design, and standard or customized graph search algorithms to generate a high-level plan.

There are three main contributions of this work. First, we provide the complete theoretical details on the requirements for the low-level control design and high-level plan, and show that these two levels operate consistently to solve the reach-avoid problem. 
Second, we formulate the parallel composition of maneuver automata in order to obtain correct-by-design motion primitives for a system composed of individual subsystems, such as in the case of multiple vehicles.
%Second, we formulate the parallel composition of maneuver automata, which is a method of composing motion primitives for individual subsystems. 
%Since the design of motion primitives can be challenging for high order systems and since many systems have a decoupled structure, such as in the case of multiple vehicles, parallel composition provides a way to describe the combined motion capabilities of the system. 
Finally, the modularity and effectiveness of our framework is experimentally validated on a group of quadrocopters in several illustrative scenarios. In particular, we feature a novel and versatile design of motion primitives based on double integrators and we show how the customizability of the high level plan generation can be used to easily trade-off solution quality with computational efficiency. This paper is an extension of our previous work \cite{VUK17}, which now supplies all the theoretical details along with proofs on correctness, the parallel composition construction, additional approaches to generate control policies, and more elaborate experimental results.

The paper is organized as follows. In the next section we highlight our contributions relative to the literature. 
In Section~\ref{sec:prob} we present a formal problem statement. The modular framework is introduced in 
Section~\ref{sec:method}. We define the output transition system, the maneuver automaton, the product automaton, 
and the high level plan, each of which contribute to realizing a solution of the motion planning problem. 
In Section \ref{sec:main}, we prove that our overall methodology solves the motion planning problem. 
In Section~\ref{sec:paracomp} we give the procedure for composing motion primitives. In Section~\ref{sec:MAexample} 
we present specific motion primitives for a double integrator system. 
In Section~\ref{sec:application} we consider several methods to generate high level plans, which are experimentally demonstrated on quadrocopters. We conclude the paper in 
Section~\ref{sec:conclusion}. 

{\bf \em Notation.} \
Let $\ZZ$ denote the integers and $\RR$ denote the real numbers. Let $| \cdot |$ denote the cardinality of a set. If $A$ is a set, we denote its power set as $2^A$. If $A$ and $B$ are sets, let $A \setminus B$ denote the usual set difference. If there are $n$ sets $A_i$, let $\prod_{i = 1}^n A_i$ denote the usual cartesian product. 
Given a function $f : A \rightarrow B$, the image of $A_1 \subset A$ under $f$ and the preimage of $B_1 \subset B$ under $f$ are defined the usual way, and are denoted as $f(A_1) \subset B$ and $f^{-1}(B_1) \subset A$, respectively.
%Let $|| \cdot ||$ denote the Euclidean norm on $\RR^n$. 
%For a set $A \subset \RR^n$, we denote its interior as $\inter(A)$. 
Let $\textrm{co}\{v_1,\ldots,v_m\}$ denote the convex hull of the vectors $v_1,\ldots,v_m \in \RR^n$. Given two vectors $v,w \in \RR^n$, we denote the component-wise multiplication (or Hadamard product) as $v \circ w$. Let $\mathcal{X}(\mathbb{R}^{n})$ denote the set of globally Lipschitz vector fields on $\RR^n$.

\section{Related Literature}

The literature on motion planning is vast and encompasses many research communities. As such, we have categorized some common approaches and discussed how they relate to our method.
%we have categorized and discussed only the most related work to our approach, and we highlight the papers that have been inspirational for our work.

\subsection{Graph Search and Trajectory Planning}

Motion planning has often been addressed as a discrete planning problem, for which many standard graph search algorithms exist \cite{LAV06}. Recent work on the multi-agent reach-avoid problem has developed novel algorithms in the context of applications such as manufacturing and warehouse automation, aiming to balance computational efficiency with solution quality. For example, a centralized approach is given in \cite{YU18}, discretizing the workspace into a lattice and using integer linear programming to minimize the total time for robots to traverse in high densities. In \cite{HALP17}, a sampling-based roadmap is constructed in the joint robot space using individual robot roadmaps, which is shown to be asymptotically optimal. Prioritized planning enables to safely coordinate many vehicles and is considered in a centralized and decentralized fashion in \cite{CAP15}. Subdimensional expansion computes mainly decentrally, but coordinates in the joint search space when agents are neighboring \cite{CHOS15}. While such approaches typically provide various theoretical guarantees on the proposed algorithms, dynamical models and application on real robotic systems is often not considered.

The modularity of our framework is complementary, as it potentially enables existing multi-agent literature on gridded workspaces to be used directly or adapted for the generation of a high-level plan when used in conjunction with our proposed formulation of motion primitives.
However, the consideration of continuous time dynamics may complicate the application of discrete planning methods in two ways. First, we must contend with constraints on successive motion primitives so that the continuous time behavior is acceptable - for example, avoiding abrupt changes in velocity. Second, we must contend with non-deterministic transitions to neighboring boxes, because motion primitives may allow more than one next box to be reached \cite{BEL08b} - for example, modeling the joint asynchronous motion capabilities of a multi-robot system. 
%In the case of non-trivial constraints or non-determinism, graph search algorithms would need to take these factors into consideration. 
%To address these factors, we have provided a specific planning algorithm based on non-deterministic Dijkstra \cite{BRO05}.

%A second inspirational paper is \cite{WOL13}, since it deals with general non-deterministic transition systems and LTL specifications. They identify the difficulty to solve general LTL control problems, so they focus on a fragment of LTL that reduces to solving a sequence of reach-avoid problems. They provide high level planning strategies that potentially can have wide practical applicability. As such, we have built our framework in order that it can be extended to address their fragment of LTL; see Remarks \ref{rem:costtogo} (iii) and \ref{rem:reach-avoidseq}. As with much of the literature on discrete planning on graphs, they assume a transition system is already known. Our approach of designing motion primitives and applying them to a gridded output space provides an explicit construction for these transition systems.

%A standard formulation for multi-robot motion planning is the formation change problem. 
Trajectory tracking methods have also been applied to the formation change problem on real vehicles with complex dynamics.
A sequential convex programming approach is given in \cite{DAND12}, which computes discretized, non-colliding positional trajectories for a modest number of quadrocopters. More recently, an impressive number of quadrocopters were coordinated in \cite{AYAN17}, by first computing a sequence of grid-based waypoints and then refining it into smoother piecewise polynomials. However, since these open-loop trajectories are computed offline, deviations from the computed trajectories could result in crashes. On the other hand, our approach is more robust as it is completely untimed, carefully monitoring the progress of vehicles over the grid in a reactive way based on the measured box transitions.

\subsection{Formal Methods}
%Several papers have been inspirational for our work. 
A growing body of research has explored the use of formal methods in motion planning.
This paper has been particularly inspired by \cite{BELTA08}, which provides a general framework for solving control problems for affine systems with LTL specifications. Their approach involves constructing a transition system over a polyhedral partition of the state space that arises from linear inequality constraints that constitute the atomic propositions of the LTL specification. Transitions between states of the transition system can occur if there exists an affine or piecewise affine feedback steering all continuous time trajectories from one polyhedral region to a contiguous one. Similar works to \cite{BELTA08} include \cite{CHOS03, HVS06, LAV09}, which consider the simpler reach-avoid problem. Single and multi-robot applications followed shortly after in \cite{PAP09} and \cite{AYAN10} respectively.

The appeal of these approaches is derived from their generality and faithful account of the low level system capabilities. On the downside, these methods generally do not scale well to larger state space dimensions, and so they would have limited applicability to large multi-robot systems. 
Our approach specializes these ideas by exploiting symmetry in the system dynamics and grid partition in order to strike a better balance between generality and computational efficiency. In particular, our feedback controllers are given as motion primitives, which can be designed independently of the obstacle and goal locations.

More recent works have also built on these formal method approaches, investigating more complex and realistic multi-robot problems. For example, service requests by multiple car robots in a city-like environment with communication constraints was considered in \cite{BELT12}. A cooperative task for ground vehicles was addressed in a distributed manner, enabling knowledge sharing amongst neighbors and reconfiguration of the motion plan in real time \cite{DIM15}. Tasks such as picking up objects are considered in conjunction with motion requirements in \cite{GAZ14}. Since these works consider only fairly simple vehicle dynamics, they place greater emphasis on the synthesis of discrete plans satisfying the task specification. On the other hand, this paper considers the simpler reach-avoid problem in order to develop a formulation of motion primitives for nonlinear systems with symmetries.

% These works each consider specific dynamics and design compatible feedback controllers, which might be simplified using our formulation of motion primitives.}
%It is mainly the complexity of the temporal specification that generally prohibits application of these methods to large robotic groups
%These works each consider specific dynamics and design compatible feedback controllers, but still do not exploit potential computational advantages that could be obtained by using motion primitives as we propose.
%Persistent surveillance for quadrocopters was formulated as a bounded LTL specification in \cite{BELT15}

\subsection{Motion Primitives}

The usage of motion primitives has become popular recently in robotics, as they serve to simplify the motion planning problem by using predefined executable motion segments. Many variations exist, which have designed motion primitives using timed reference trajectories to control a formation of quadrocopters \cite{KUM14}, paths on a state space lattice for a mobile robots \cite{MIH09, KOE14}, and funnels in the state space centered about a reference trajectory for a car \cite{SATT14} and a small airplane \cite{TED17}.

We have been inspired by ideas in \cite{FRAZ05}, from which we borrowed the term ``maneuver automaton". They define a motion primitive either as an equivalence class of trajectories or a timed maneuver between two classes, whereas we define a motion primitive as a feedback controller over a polyhedral region in the state space. In our formulation, concatenations between motion primitives are possible only across contiguous boxes in the output space, which provides a strict safety guarantee during concatenation. Moreover, this enables our approach to simplify obstacle avoidance to a discrete planning problem over safe boxes as in \cite{KOE14}, bypassing the need to concatenate motion primitive trajectories using numerical optimization techniques as in \cite{FRAZ05}.

Our presentation of the maneuver automaton gives explicit constraints on the design of motion primitives so that they can used reliably for high level planning.
%While providing a general constructive methodology for maneuver automata of arbitrary nonlinear systems with symmetries is beyond the scope of this work, existing tools from reachability analysis can be employed. To alleviate the difficulty of constructing maneuver automata, we}
We have also introduced the notion of parallel composition of maneuver automata to build motion primitives for multi-robot systems. While our construction resembles existing methods of parallel composition \cite{WON15, PAP01}, we additionally prove that our construction preserves desired properties that enable consistency between the low and high levels. 
To the authors' best knowledge, this paper is the first rigorous treatment of feedback-based motion primitives defined on a uniformly gridded output space. 

\section{Problem Statement} 
\label{sec:prob}

Consider the general nonlinear control system
\begin{equation}
\label{eq:thesystem}
\dot{x} = f(x,u), \;\;\;\;\; y = h(x),
\end{equation}
where $x \in \RR^n$ is the state, $u \in \RR^\mu$ is the input, and $y \in \RR^p$ is the output. Let 
$\phi(\cdot,x_0)$ and $y(\cdot,x_0)$ denote the state and output trajectories of \eqref{eq:thesystem} 
starting at initial condition $x_0 \in \RR^n$ and under some open-loop or feedback control. 

Let $\cP \subset \RR^p$ be a feasible set in the output space and let $\cG \subset \cP$ be a goal set. In multi-vehicle motion planning contexts, $\cP$ represents the feasible joint output configurations of the system, which can arise from specifications involving obstacle avoidance, collision avoidance, communication constraints, and others. We consider the following problem.

\begin{problem}[Reach-Avoid]
\label{prob:reachavoid} 
We are given the system \eqref{eq:thesystem}, a non-empty feasible set $\cP \subset \RR^p$ and a non-empty goal set $\cG \subset \cP$. 
Find a feedback control $u(x)$ and a set of initial conditions $\cX_0 \subset \RR^n$ such that for each $x_0 \in \cX_0$ we have
\begin{itemize}
\item[(i)]
{\bf Avoid}: 
$y(t,x_0) \not \in \RR^p \setminus \cP$ for all $t \geq 0$,
\item[(ii)]
{\bf Reach}:
there exists $T \geq 0$ such that for all $t \geq T$, $y(t,x_0) \in \cG$.
\end{itemize}
\end{problem}

We make an assumption regarding the outputs of the system \eqref{eq:thesystem} in order to exploit symmetry; 
see \cite{FRAZ05} for an exposition on nonlinear control systems with symmetries.

\begin{assum}
\label{assum:symmetry}
First, we assume that there is an injective map $o : \{1, \ldots, p\} \rightarrow \{1, \ldots, n\}$ associating each output to a distinct state, so that $h(x) = (x_{o(1)}, \ldots, x_{o(p)})$. 
We define the (injective) insertion map $h^{-1}_o : \RR^p \rightarrow \RR^n$ as $h^{-1}_o(y) = x$, which satisfies $h(x) = y$ and $x_i = 0$ for all $i \in \{1, \ldots, n\} \setminus \{o(1), \ldots, o(p)\})$.
Second, we assume that the system has a {\em translational invariance} with respect to its outputs. That is, for all $x\in \RR^n$, $u\in \RR^{\mu}$ and $y \in \RR^p$, we have $f(x,u) = f(x + h^{-1}_o(y), u)$.
\tqed
\end{assum}

The assumption that the outputs of the system are a subset of the states is used in our framework to be able to design feedback 
controllers in the full state space that achieve desirable behavior in the output space. The second statement says that the vector 
field is invariant to the value of the output. In the literature this condition is called a {\em symmetry of the system} or 
{\em translational invariance}. This assumption is satisfied for many robotic systems, for example, when the outputs are positions. 
Also, we will see in Section~\ref{sec:MAexample} that it significantly simplifies our control design. 

\section{Modular Framework} 
\label{sec:method}

\begin{figure}%[t]
\centering%
\includegraphics[width=1\linewidth,trim=0cm 0cm 0cm 0cm, clip=true]{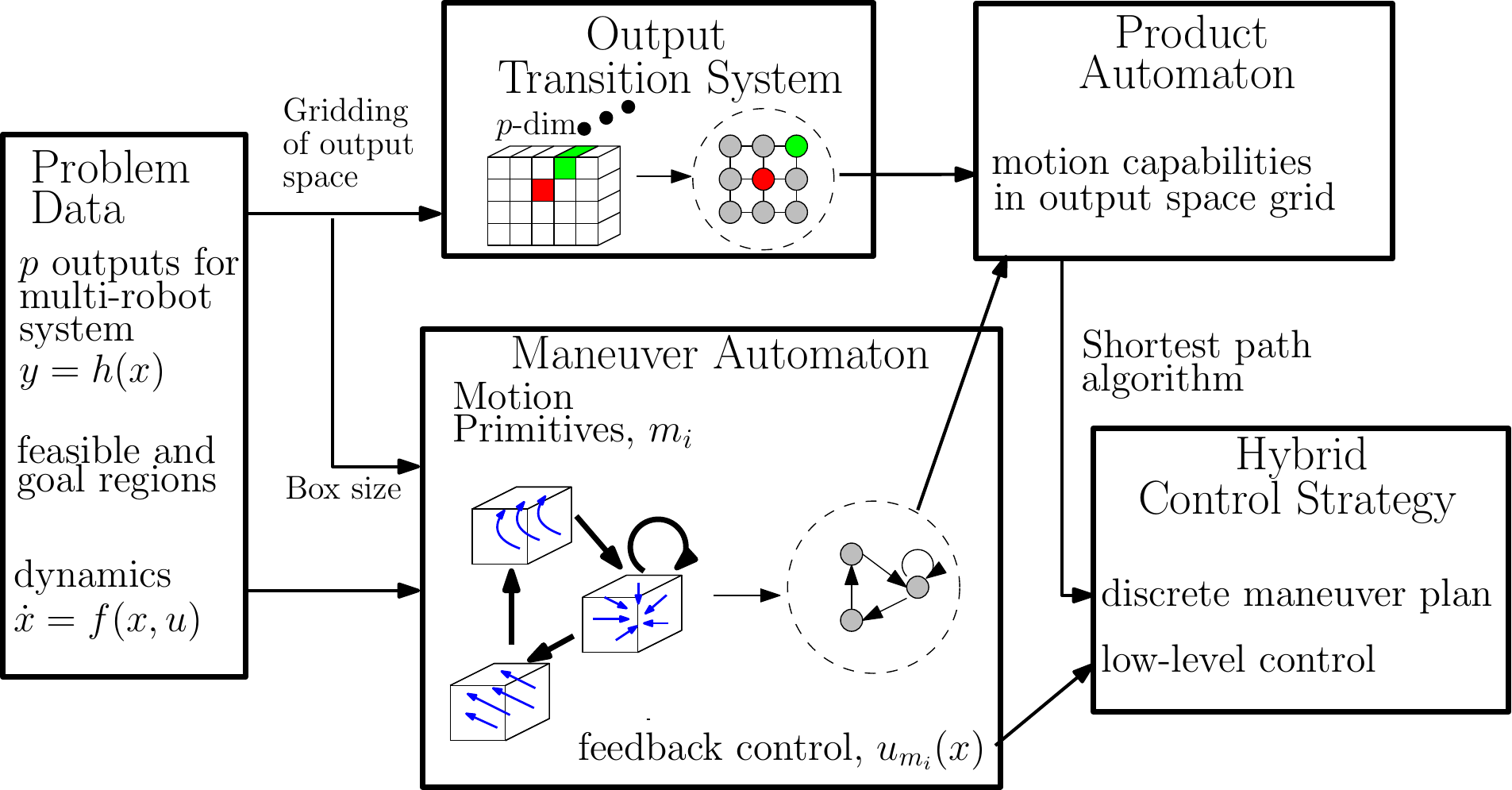}
\caption{Our modular framework consists of five modules.} 
\vspace{-2mm}
\label{fig:methodology}%
\end{figure}

In this section we present our methodology to solve the motion planning problem in the form of an architecture
that breaks down Problem~\ref{prob:reachavoid}. This architecture consists of five main modules, as depicted in
Figure~\ref{fig:methodology}.
\begin{itemize}

\item 
The {\em Problem Data} include the system \eqref{eq:thesystem} with $p$ outputs satisfying 
Assumption~\ref{assum:symmetry} and a reach-avoid task to be executed in the output space.

\item 
The {\em Output Transition System (OTS) is a directed graph whose nodes (called 
locations) represent $p$-dimensional boxes on a gridded output space and whose edges describe which boxes in the output space are contiguous.}

\item 
The {\em Maneuver Automaton} (MA) is a hybrid system whose modes correspond to so-called 
motion primitives. Each motion primitive is associated with a closed-loop vector field by applying a feedback law 
to \eqref{eq:thesystem}. The edges of the MA represent feasible successive motion primitives. Each motion primitive 
generates some desired behavior of the output trajectories of the closed-loop system over a box in the output space. 
Because of the uniform gridding of the output space into boxes and because of the symmetry in the outputs described 
in Assumption \ref{assum:symmetry}, motion primitives can be designed over only one canonical box $Y^*$. 

\item 
The {\em Product Automaton} (PA) is a graph which is the synchronous product of the OTS and the discrete
part of the MA. It represents the combined constraints on feasible motions in the output space and feasible successive 
motion primitives. 

\item The {\em Hybrid Control Strategy} is a combination of low level controllers obtained from the design of 
motion primitives, and a high level plan on the product automaton.
%obtained by applying a shortest path algorithm adapted to non-deterministic graphs on the PA \cite{BRO05}.

\end{itemize}
Next we describe in greater detail the OTS, MA, and PA. 

\subsection{Output Transition System} \label{sec:OTS}

\begin{figure}%[t]
\centering%
\includegraphics[width=1.0\linewidth,trim=0cm 0cm 0cm 0cm, clip=true]{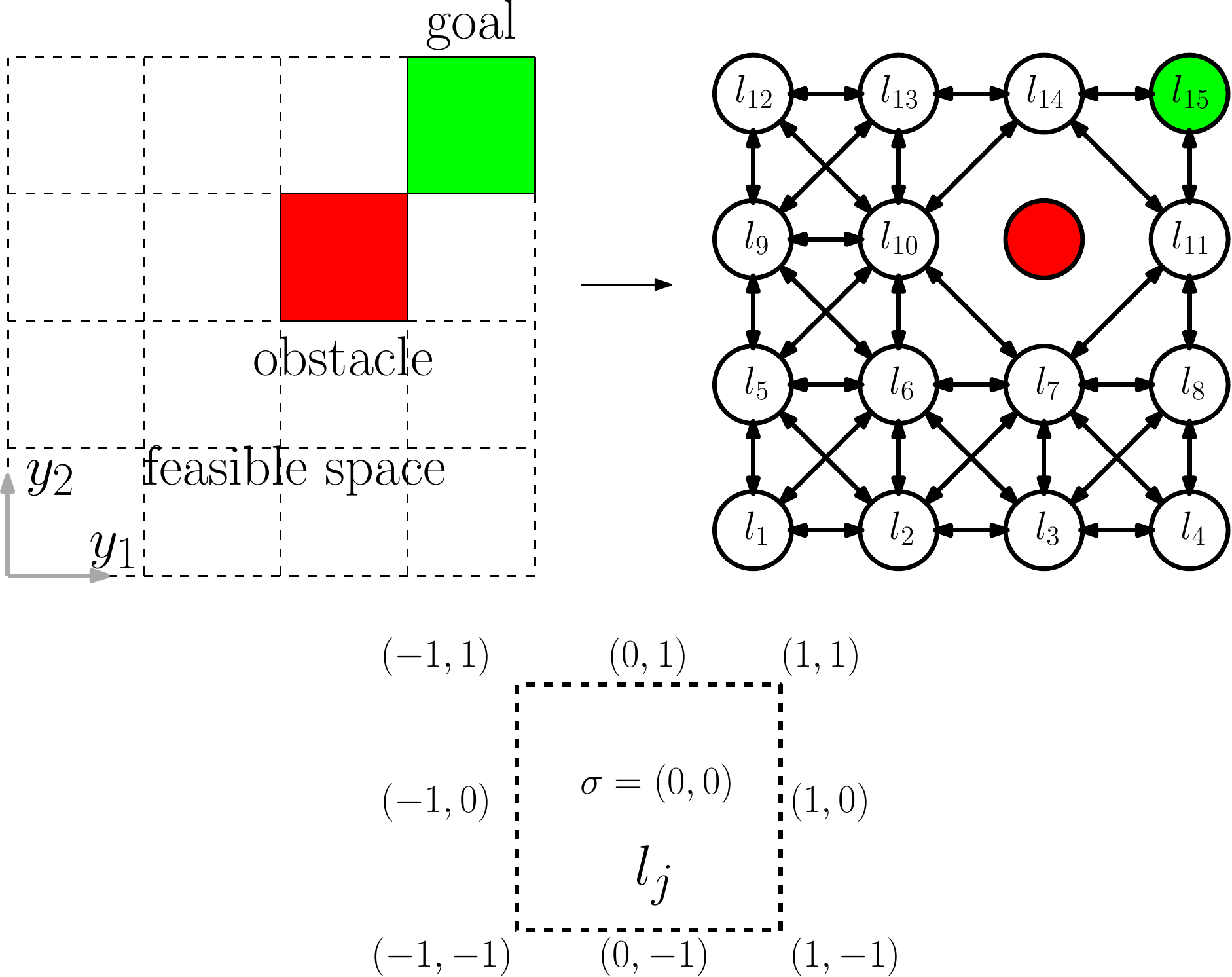}
\caption{A two output ($p = 2$) example of a reach-avoid task. Shown on the left is the feasible space $\cP$ consisting of 
15 non-obstacle boxes (not red) and the goal region $\cG$ (green). The output transition system (OTS), which 
abstracts the box regions and their neighbour connectivity, is shown on the right. Shown below, the 
possible offsets towards a neighbouring box are labelled using $\Sigma =  \{-1,0,1\}^2$.}%
\vspace{-2mm}
\label{fig:OTSsquares}%
\end{figure}

The OTS provides an abstract description of the {\em workspace} or {\em output space} associated with the system 
\eqref{eq:thesystem}. It serves to capture the feasible motions of output trajectories of the system 
\eqref{eq:thesystem} in a gridded output space, as in Figure~\ref{fig:OTSsquares}. 
Specifically, we partition the output space into $p$-dimensional boxes with lengths $d = (d_1, \ldots, d_p)$, where $d_i > 0$ is the length of $i$-th edge. We use a finite number of boxes to under-approximate the feasible set $\cP$. Enumerating the boxes as $\{ Y_1, \ldots, Y_{n_L} \}$, the $j$-th box can be expressed in the form
$Y_j := \prod_{i = 1}^p \left[ \eta_{ji} d_i, (\eta_{ji} + 1) d_i \right]$,
where $\eta_{ji} \in \ZZ$, $i = 1,\ldots,p$ are constants. 
%Thus we have that
We require that
$\bigcup_{j = 1}^{n_L} Y_j \subset \cP$.
Among these boxes, we assume there is a non-empty set of indices $I_g \subset \{1, \ldots, n_L\}$, so that 
we may similarly under-approximate the goal region as 
$\bigcup_{j \in I_g} Y_j \subset \cG \subset \cP$.
We define a canonical $p$-dimensional box with edge lengths $d_i > 0$ given by
$Y^* = \prod_{i = 1}^p [0, d_i]$.
Each box $Y_j$, $j = 1,\ldots, n_L$ is a translation of $Y^*$ by
an amount $\eta_{ji} d_i$ along the $i$-th axis.

\begin{defn}
Given the lengths $d$ and a non-empty goal index set $I_g$, an {\em output transition system} (OTS) is a tuple $\cAOTS = ( \LOTS, \Sigma, E_{\OTS}, \LOTS^g )$
with the following components:
\begin{description}
\item[State Space]
$\LOTS := \{l_1,\ldots, l_{n_L} \} \subset \ZZ^p$ is a finite set of nodes called {\em locations}.
Each location $l_j \in \LOTS$ is associated with a safe box $Y_j\subset \cP $ in the output space and hence we write $l_j = (\eta_{j1}, \ldots, \eta_{jp})$.

\item[Labels]
$\Sigma := \{-1, 0, 1\}^p \subset \ZZ^p$ is a finite set of {\em labels}.
A label $\sigma \in \Sigma$ is used to identify the offset between neighbouring boxes.

\item[Edges]
$\EOTS \subset \LOTS \times \Sigma \times \LOTS$ is a set of {\em directed edges} where $(l_{j},\sigma,l_{j'}) \in \EOTS$ if 
$j \neq j'$, $Y_{j} \cap Y_{j'} \neq \emptyset$, and $\sigma = l_{j'} - l_j \in \Sigma$. Thus, for each $i = 1, \ldots, p$, 
the neighbouring box $l_{j'}$ is either one box to the left ($\sigma_i = -1$), the same box ($\sigma_i = 0$), or one box to the 
right ($\sigma_i = 1$). In this manner $\sigma$ records the offset between contiguous boxes.

\item[Final Condition]
$\LOTS^g = \{ l_j \in \LOTS ~|~ j \in I_g \}$ denotes the set of locations associated with goal boxes.

\tqed
\end{description}
\end{defn}

\begin{remark}
We observe that the OTS is deterministic. That is, for a given $l \in \LOTS$ and $\sigma \in \Sigma$, there is at most one $l' \in \LOTS$ such that $(l, \sigma, l') \in \EOTS$. This follows immediately from the fact that $\sigma = l' - l$ records the offset between the neighbouring boxes.
\end{remark}

Figure \ref{fig:OTSsquares} shows a sample OTS for a simple scenario. The OTS locations are associated with 15 feasible boxes, including a goal box for the reach-avoid task. The OTS edges are shown as bidirectional arrows; for example, interpreting $l_1 = (0,0)$ and $l_6 = (1,1)$ on the grid, then $e = (l_6, (-1, -1), l_1) \in \EOTS$.

\subsection{Maneuver Automaton} 
\label{sec:MA}

The \textit{maneuver automaton} (MA) is a hybrid system consisting of a finite automaton and continuous time dynamics 
in each discrete state. The discrete states of the finite automaton correspond to {\em motion primitives}, while 
transitions between discrete states correspond to the allowable transitions between motion primitives. The continuous 
time dynamics are given by closed-loop vector fields \eqref{eq:thesystem} with a control law designed based on 
reach control theory (any other feedback control design method can be used).

Before presenting the MA, we first explain how this module is used in the overall framework. To solve 
Problem~\ref{prob:reachavoid}, we assign motion primitives to the boxes $Y_j$ of the partitioned output space such that obstacle regions are avoided and the goal region is eventually reached. The discrete part of the MA encodes the constraints on successive motion primitives. Such constraints might arise from a non-chattering requirement, continuity requirement, or requirement on correct switching between regions of the state space. A dynamic programming algorithm for assignment of motion primitives on boxes is addressed in Section~\ref{sec:dpp}; the salient point about this algorithm at this stage is that it only uses the discrete part of the MA.

In contrast, the continuous time part of the MA is used both for simulation of the closed-loop dynamics to verify that the motion primitives are well designed, as well as for the implementation of the low level feedback in real-time. The motion primitives are defined only on the canonical box $Y^*$ to simplify the design. This simplification is possible because of the translational symmetry provided by Assumption~\ref{assum:symmetry} and the fact that each box $Y_j$ is a translation of $Y^*$. 
In simulation, a given motion primitive can cause output trajectories to reach certain faces of $Y^*$. If a face is reached, the output trajectory is interpreted as being reset to the opposite face and another motion primitive is selected to be implemented over $Y^*$ (of course, the real experimental output trajectories do not undergo resets but move continuously from box to box in the output space). 
The selection of the next motion primitive is constrained by a combination of the previous motion primitive and the face of $Y^*$ that is reached. 
The discrete transitions in the MA encode these constraints. 

\begin{defn}
Consider the system \eqref{eq:thesystem} satisfying Assumption \ref{assum:symmetry} and the box $Y^*$ with lengths $d$. 
The {\em maneuver automaton} (MA) is a tuple $\cHMA = ( \QMA, \Sigma, \EMA, \XMA, \IMA, \GMA, \RMA, \QMA^0)$, where

\begin{description}%[labelindent=1cm] %[style=multiline,leftmargin=3cm,font=\normalfont]

\item[State Space]
$\QMA = M \times \RR^n$ is the hybrid state space, where $M = \{ m_1, \ldots, m_{n_M} \}$ is a finite set
of nodes, each corresponding to a motion primitive.

\item[Labels]
$\Sigma$, the same labels used in the OTS.

\item[Edges]
$\EMA \subset M \times \Sigma \times M$ is a finite set of edges.

\item[Vector Fields]
$\XMA : M \rightarrow \mathcal{X}(\mathbb{R}^{n})$ is a function assigning a globally Lipschitz closed-loop vector 
field to each motion primitive $m \in M$. That is, for each $m \in M$, we have $\XMA(m) = f( \cdot, u_m(\cdot) )$ 
where $u_m(\cdot)$ is a feedback controller associated with $m \in M$.

\item[Invariants]
$\IMA : M \rightarrow 2^{\RR^n}$ assigns a bounded invariant set $\IMA(m)$ to each $m \in M$. We impose that 
$\IMA(m) \subset h^{-1}(Y^*)$. The set $\IMA(m)$ defines the region on which the vector field $\XMA(m)$ is defined. 
Note that there is no requirement that the invariant is a closed set.

\item[Enabling Conditions]
$\GMA : \EMA \rightarrow \{ g_e \}_{e \in \EMA}$ assigns to each edge $e = (m, \sigma, m') \in \EMA$, a non-empty enabling or 
guard condition $g_e \subset \RR^n$. We require that $g_e \subset \IMA(m)$. We make an additional requirement 
that $g_e$ lies on a certain face of $Y^*$ determined by the label $\sigma = (\sigma_1,\ldots,\sigma_p) \in \Sigma$. Defining the 
face associated with $\sigma$ as
\begin{eqnarray*}
\cF_{\sigma} = \left \lbrace y \in Y^* ~\middle |~
\begin{cases}
y_i = 0,        & \text{~if~} \sigma_i = -1 \\
y_i = d_i,     & \text{~if~} \sigma_i = 1 \\
y_i \in [0, d_i], & \text{~if~} \sigma_i = 0
\end{cases}
\right \rbrace \,,
\end{eqnarray*}
we require that also $g_e \subset h^{-1}(\cF_{\sigma})$. 

\item[Reset Conditions]
$\RMA: \EMA \rightarrow \{ r_e \}_{e \in \EMA}$ assigns to each edge $e = (m,\sigma,m')\in \EMA$
a reset map $r_e : \RR^n \rightarrow \RR^n$. We require that $r_e(x) = x - h^{-1}_o(d \circ \sigma)$, where $\circ$ is the Hadamard product.
This definition says that the $i$-th output component is reset to the right face of $Y^*$, $x_{o(i)} = d_i$, 
if $\sigma_i = -1$, reset to the left face $x_{o(i)} = 0$ if $\sigma_i = 1$, and unchanged otherwise.
Overall, resets of states are determined by the event $\sigma \in \Sigma$ and only affect the 
output coordinates in order to maintain output trajectories inside the canonical box $Y^*$. 

\item[Initial Conditions]
$\QMA^0 \subset \QMA$ is the set of initial conditions given by 
$\QMA^0 = \{ (m,x) \in \QMA ~|~ x \in \IMA(m) \}$. 

\tqed
\end{description}
\end{defn}

\begin{example} \label{ex:MAex}
Suppose the system is a double integrator and the first state is the translationally invariant output $y$. The box $Y^*$ is simply a segment. Let $M = \{ \sH, \sF, \sB \}$, where {\em Hold} ($\sH$) stabilizes $y$, {\em Forward} ($\sF$) increases $y$, and {\em Backward} ($\sB$) decreases $y$. Referring to Figure~\ref{fig:MATrans2}, if $\sF$ is the current motion primitive and $y$ reaches the right face of $Y^*$, then the event $1 \in \Sigma$ occurs and we may select $\sH$ or $\sF$ as the next motion primitive.
To correctly implement the discrete evolution of the MA in the continuous state space, an invariant and feedback control must be associated with each motion primitive, while an enabling and reset condition must be associated with each edge; see Figure \ref{fig:DoubleVF}.  Formal details are given in Section \ref{sec:MAexample}.
\end{example}

\begin{figure}%[t]
\centering%
\includegraphics[width=0.7\linewidth,trim=0cm 0cm 0cm 0cm, clip=true]{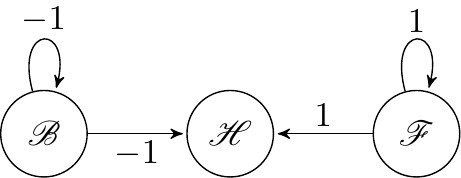}
\caption{The maneuver automaton edges $\EMA$ for the double integrator dynamics with $p = 1$.
There are three motion primitives: Hold ($\mathscr{H}$), Forward ($\mathscr{F}$), and Backward ($\mathscr{B}$). 
}%
\vspace{-2mm}
\label{fig:MATrans2}%
\end{figure}

\begin{figure}[t]
\centering%
\includegraphics[width=\linewidth,trim=3.5cm 3.5cm 2cm 2cm, clip=true]{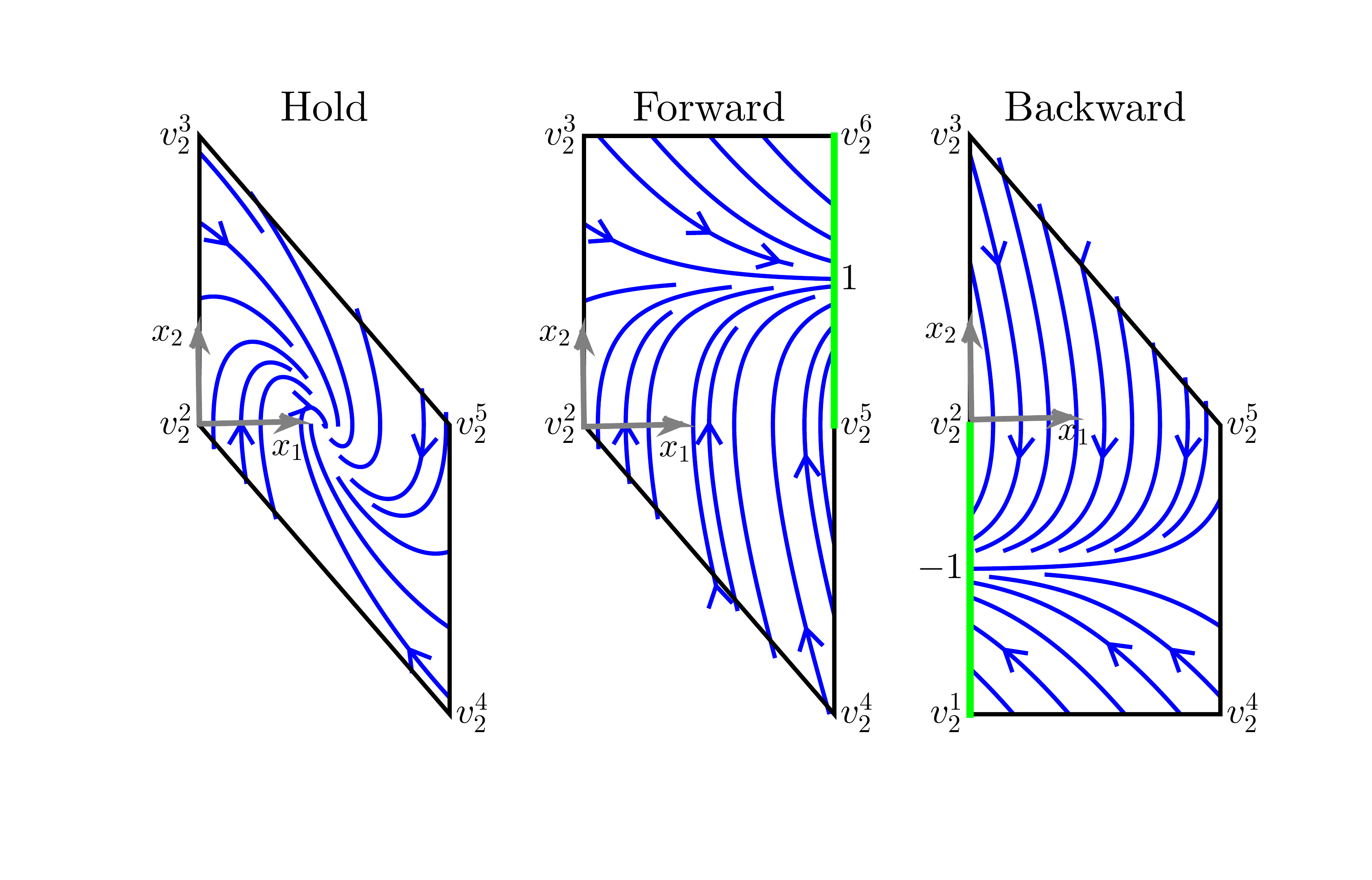} \\
\caption{The closed-loop vector fields in the $(x_1,x_2)$ position-velocity state space for the Hold, Forward, and Backward motion primitives.
}%
\vspace{-2mm}
\label{fig:DoubleVF}%
\end{figure}

We now formulate assumptions on the motion primitives so that correct continuous time behavior is 
ensured at the low level for consistency with the high level. For each $m \in M$, define the set 
of possible events as
\begin{equation}
\label{eq:SMA}
\SMA(m) := \{ \sigma\in \Sigma ~|~ (\exists m'\in M) (m,\sigma,m') \in \EMA \} \,.
\end{equation}

\begin{assum}
\label{assum:MP}
\begin{itemize}
\item[]
\item[(i)] 
For all $m \in M$, $\varepsilon:=(0,\ldots,0)\not\in \SMA(m)$.

\item[(ii)]
For all $e_1,e_2 \in \EMA$ such that $e_1 = (m_1, \sigma, m_2)$ and $e_2 = (m_1, \sigma, m_3)$, $g_{e_1} = g_{e_2}$.

\item[(iii)]
For all $e_1, e_2 \in \EMA$ such that $e_1 = (m_1, \sigma_1, m_2)$ and $e_2 = (m_1, \sigma_2, m_3)$, if $\sigma_1 \neq \sigma_2$, 
then $g_{e_1} \cap g_{e_2} = \emptyset$.

\item[(iv)] 
For all $e_1, e_2 \in \EMA$ such that $e_1 = (m_1, \sigma_1, m_2)$ and $e_2 = (m_2, \sigma_2, m_3)$, 
$r_{e_1}(g_{e_1}) \cap g_{e_2} = \emptyset$.

\item[(v)]
For all $e = (m_1,\sigma,m_2) \in \EMA$, $r_e(g_e) \subset \IMA(m_2)$. 

\item[(vi)] 
For all $m\in M$, if $\SMA(m) = \emptyset$ then for all $x_0 \in \IMA(m)$ and $t\geq 0$, 
$\phiMA(t,x_0)\in \IMA(m)$.

\item[(vii)]
For all $m \in M$, if $\SMA(m) \neq \emptyset$, then for all $x_0 \in \IMA(m)$ there exist
(a unique) $\sigma \in \SMA(m)$ and (a unique) $T \geq 0$ such that for all $e = (m,\sigma,m') \in \EMA$ 
and for all $t \in [0,T]$, $\phiMA(t,x_0) \in \IMA(m)$ and $\phiMA(T,x_0) \in g_{e}$.

\tqed
\end{itemize}
\end{assum}

Condition (i) disallows tautological chattering behavior that arises by erroneously interpreting continuous evolution of 
trajectories in the interior of $Y^*$ as ``discrete transitions'' of the MA (see Section \ref{sec:main} for definitions).
Condition (ii) imposes that guard sets are independent of the next motion primitive. Since guard sets arise as the set of 
exit points of closed-loop trajectories from $Y^*$ under a given motion primitive, it is reasonable that these exit points 
should depend only on the current motion primitive $m \in M$, and not on the choice of next motion primitive. 
Condition (iii) imposes that all guard sets corresponding to different labels are non-overlapping.
This ensures that when the continuous trajectory reaches a guard $g_e$, then it is unambiguous which edge of the MA is
taken next; namely $e \in \EMA$. Conditions (v), (vi), and (vii) are placed to guarantee that the MA is non-blocking. 
These conditions are based on known results in the literature \cite{LYG99}; see Lemma~\ref{lem:nb}.
In order for condition (vii) to make sense, there must exist a unique label $\sigma \in \Sigma$ and a unique time 
$T \ge 0$ for an MA trajectory to reach a guard set. First, we have uniqueness of solutions since the vector fields are globally Lipschitz. Second, 
the unique MA trajectory can only reach one guard set by condition (iii); this in turn means there is a unique $\sigma$. 
Obviously there exists a unique time to reach the guard set. Conditions (vi) and (vii) work together to state that 
either all trajectories do not leave, or all trajectories do eventually leave. 
Referring to Figure~\ref{fig:DoubleVF}, all closed-loop state trajectories within the invariant of $\sF$ reach the guard set shown in green on the right. For either choice of next feasible motion primitive, $\sH$ or $\sF$, trajectories enter the next invariant on the left due to the reset.
Finally, condition (iv) eliminates potential chattering Zeno behavior, see Remark \ref{rem:Zeno1}. 

\begin{remark}
\label{rem:MA}
We make several further observations about the MA.

%(i) Motion primitives are defined only on $Y^*$ so trajectories of the MA undergo resets when MA output trajectories reach a guard lying on a face of $Y^*$ so that they will continue to evolve on $Y^*$. In contrast, the trajectories of \eqref{eq:thesystem} and of the real physical system do not undergo resets. 

(i) The MA is non-deterministic in the sense that given $m \in M$ and $\sigma \in \Sigma$, 
there may be multiple $m' \in M$ such that $(m,\sigma,m') \in \EMA$. The discrete part of the MA
is non-deterministic in a second sense: for each $m \in M$, the cardinality of the set 
$\SMA(m)$ may be greater than one. The latter situation corresponds to the fact that for different 
initial conditions $x_1, x_2 \in \IMA(m)$ of the continuous part, the associated output trajectories can reach 
different guard sets. In essence, which guard is enabled is interpreted, at the high level, as an 
uncontrollable event \cite{WON15}. 
%Referring to Figure~\ref{fig:ctrStrat}, the motion primitive $(\sF, \sF)$ may cause the events $(1,0)$, $(0,1)$, or $(1,1) \in \Sigma$.
Remark~\ref{rem:PAnondet} further illustrates these two types of 
non-determinism in the case of the PA. 

(ii) The set of events $\Sigma$ in the MA correspond to the same events $\Sigma$ in the OTS. This correspondence is used 
in the product automaton PA, described in the next section, to synchronize transitions in the MA with transitions in the OTS. 
The interpretation is that when a continuous trajectory of the MA (over the box $Y^*$) undergoes a reset with the label 
$\sigma \in \Sigma$, the associated continuous trajectory of \eqref{eq:thesystem} in the box $Y_j$ enters a neighboring 
box $Y_{j'}$ with the offset $\sigma = l_{j'} - l_j$. Obviously, this interpretation assumes
that the vector of box lengths $d$ is the same in both OTS and MA. 
%$d > 0$ is the same in both OTS and MA. 
\tqed

\end{remark}

\subsection{Product Automaton} \label{sec:product}

In this section we introduce the {\em product automaton} (PA). It is constructed as the synchronous product of the
OTS and the discrete part of the MA, namely $(M, \Sigma, \EMA)$. The purpose of the PA is to merge the constraints on 
successive motion primitives with the constraints on transitions in the OTS in order to enforce feasible and safe motions. 
As such, it captures the overall 
feasible motions of the system -- any high level plan must adhere to these feasible motions. 

\begin{defn}
We are given an OTS $\cAOTS$ and an MA $\cHMA$ satisfying Assumption~\ref{assum:MP}. We define the {\em product automaton} 
(PA) to be the tuple $\cAPA = ( \QPA, \Sigma, \EPA, \QPA^f )$, where

\begin{description}

\item[State Space]
$\QPA \subset  \LOTS \times M$ is a finite set of PA states. A PA state $q = (l,m) \in \QPA$ satisfies the 
following: if there exists $\sigma\in\Sigma$ and $m'\in M$ such that $(m,\sigma,m')\in\EMA$, then there exists 
$l'\in \LOTS$ such that $(l,\sigma,l')\in \EOTS$. That is, $(l,m) \in \QPA$ if all faces that can be reached by 
motion primitive $m \in M$ lead to a neighboring box of the box associated with location $l \in \LOTS$ of the OTS.

\item[Labels]
$\Sigma$ is the same set of labels used by the OTS and the MA.

\item[Edges]
$\EPA \subset \QPA \times \Sigma \times \QPA$ is a set of directed edges defined according to the following
rule. Let $q = (l,m) \in \QPA$, $q' = (l',m') \in \QPA$, and $\sigma \in \Sigma$.
If $(l,\sigma,l') \in \EOTS$ and $(m,\sigma,m')\in \EMA$, then $(q,\sigma,q') \in \EPA$.

\item[Final Condition]
$\QPA^f \subset \LOTS^g \times M$ is the set of final PA states.

\tqed
\end{description}
\end{defn}

\begin{remark} 
\label{rem:PAnondet}
Formally an automaton is said to be non-deterministic if there exists a state with more than one outgoing edge 
with the same label. The PA is non-deterministic. First, consider a PA state $q = (l,m) \in \QPA$. Because the MA 
allows for more than one feasible next motion primitive $m'$ such that $(m,\sigma,m') \in \EMA$,  
the PA will also have multiple next PA states $q' = (l', m')$ such that $(q,\sigma,q') \in \EPA$. 
Second, there can be multiple possible labels $\sigma \in \Sigma$ such that $e = (q,\sigma,q') \in \EPA$ for some
$q' \in \QPA$. Thus, the PA inherits the two types of non-determinism of the MA that we discussed in Remark~\ref{rem:MA}. 
For example, consider the PA fragment in Figure~\ref{fig:PA_nondet}. For the first type of non-determinism, 
observe that there are two PA edges $(q_1, \sigma_1, q_2) \in \EPA$ and $(q_1, \sigma_1, q_3) \in \EPA$ with 
the same label. For the second type, observe that there are two possible events $\sigma_1, \sigma_2 \in \Sigma$ 
from $q_1$, each with its own set of PA edges. Note also some additional structure: since the OTS is 
deterministic, the box state is $l_2$ in both $q_2$ and $q_3$, corresponding to the OTS edge 
$(l_1, \sigma_1, l_2) \in \EOTS$. 
\tqed
\end{remark}

\begin{figure}
\centering%
\includegraphics[width=0.7\linewidth,trim=0cm 0cm 0cm 0cm, clip=true]{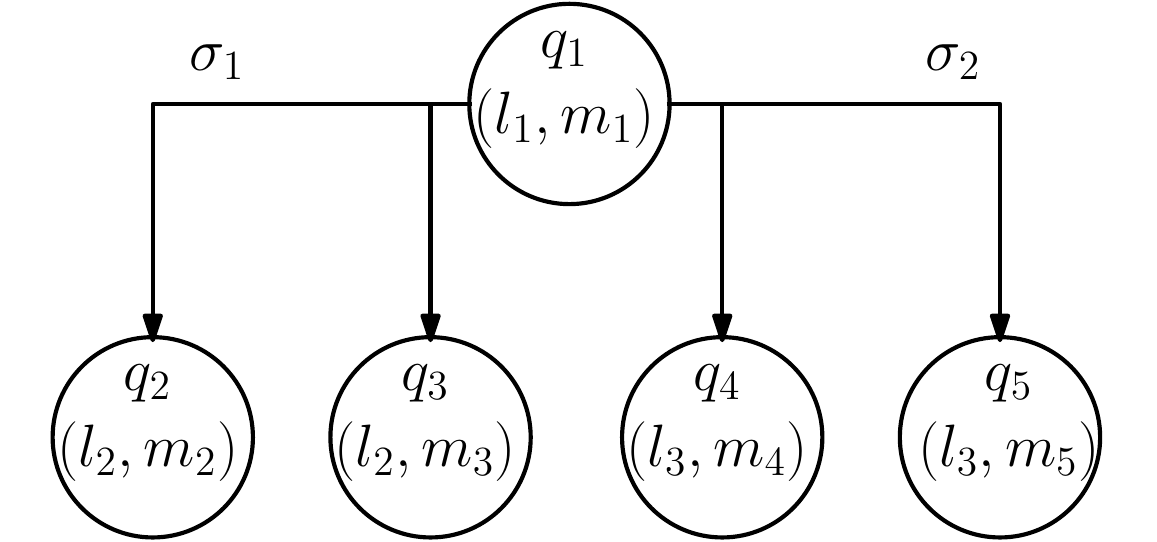}
\caption{A fragment of a generic PA, showing a state and its neighbours.
}%
\vspace{-2mm}
\label{fig:PA_nondet}%
\end{figure}

\subsection{High-Level Plan} \label{sec:dpp}

In this section we formulate the notion of a control policy on the PA, which gives a rule for selecting subsequent PA states by choosing the next motion primitive. Informally, the objective of the high level plan is to produce a control policy and find a set of initial PA states such that a goal PA state is eventually reached.
%In this section we formulate a high level plan on the PA.  Informally, the objective of the high level plan is to find a set of initial PA states and to develop a rule for selecting subsequent PA states (effectively by choosing the next motion primitive) such that a goal PA state is eventually reached. 
To this end, in this section we also develop a Dynamic Programming Principle (DPP) suitable for use on the PA. 
Because of the two types of non-determinism of the PA, existing 
algorithms cannot be applied directly \cite{BRO05,WOL13}. By adapting the algorithm in \cite{BRO05}, we 
obtain two formulations of the DPP, one of which is more computationally 
efficient as it exploits certain structure in the PA; further details are provided in 
Remark~\ref{rem:DPP}. 

First some notation will be useful. Recall from \eqref{eq:SMA}, given $m\in M$, $\SMA(m)$ is the set of 
all labels $\sigma \in \Sigma$ on outgoing edges $e \in \EMA$ starting at $m$. 
Similarly, $\SPA(q)$ is the set of all labels $\sigma \in \Sigma$ on outgoing edges 
$e \in \EPA$ starting at $q$. That is,
\[
\SPA(q) := \{ \sigma \in \Sigma ~|~ (\exists q' \in \QPA) (q,\sigma,q') \in \EPA \} \,.
\]

Now we formalize the semantics of the PA.
A {\em state} of the PA is a pair $q = (l,m) \in \QPA$ where $l \in \LOTS$ is a location in
the OTS and $m \in M$ is a motion primitive. 
A {\em run} $\pi$ of $\cAPA$ is a finite or infinite sequence of states $\pi = q^0 q^1 q^2 \dots$, with
$q^i = (l^i,m^i) \in \QPA$ and for each $i$, there exists $\sigma^i \in \SPA(q^i)$
such that $(q^i,\sigma^i,q^{i+1}) \in \EPA$. We define the length of a run to be $n_{\pi}$; 
for infinite runs $n_\pi$ is defined to be $\infty$. 
We consider a subset of runs $\PiPA(q)$ starting at $q \in \QPA$ that satisfy one further property. 
If the run $\pi$ is infinite, then $\pi \in \PiPA(q)$ if $q^0 = q$. 
Instead if the run $\pi$ is finite, then $\pi \in \PiPA(q)$ if $q^0 = q$ and additionally, 
$\SPA( q^{n_{\pi}} ) = \emptyset$. 
It is the latter requirement -- that the last PA state of a finite run may not have outgoing edges in the PA -- which is of interest. The interpretation is that we regard the event labels between PA states as uncontrollable, so if any event is possible, then it must occur eventually. Thus without loss of generality, each run $\pi = q^0 q^1 \dots$ is the prefix of a run $\pi' \in \PiPA(q^0)$. Further elaboration is given in Remark~\ref{rem:costtogo}~(ii).

Given $q \in \QPA$ and $\sigma \in \SPA(q)$, the set of {\em admissible motion primitives} is 
\[
\cM(q,\sigma) := \{ m' \in M ~|~ (\exists q' = (l',m')) ~ (q,\sigma,q') \in \EPA \} \,.
\]
More generally, given $q \in \QPA$ and $\SPA(q) = \{ \sigma_1, \ldots, \sigma_k \}$, the set of
admissible motion primitives at $q$ is
\[
\cM(q):= \{ (m_1, \ldots, m_k) ~|~ m_i\in \cM(q,\sigma_i),\, i=1,\dots,k \},
\]
Next we introduce the notion of a control policy. Given $q \in \QPA$ and 
$\SPA(q) = \{ \sigma_1, \ldots, \sigma_k \}$, an {\em admissible control assignment} at $q$ 
is a vector
\[
c(q) = (c(q,\sigma_1), \ldots, c(q,\sigma_k)) \,,
\]
where $c(q,\sigma_i) \in \cM(q,\sigma_i)$, or equivalently $c(q) \in \cM(q)$. 
Notice that $c(q)$ is a vector whose dimension varies as a function of the cardinality of the set $\SPA(q)$. 
An \textit{admissible control policy} $c: \QPA \times \Sigma \rightarrow M$ is a map that assigns an
admissible control assignment at each $q \in \QPA$. Thus, for each $q \in \QPA$ and $\sigma \in \SPA$,
$c(q,\sigma) \in \cM(q,\sigma)$. The set of all admissible control policies is denoted by $\cC$.

Consider an admissible control policy $c \in \cC$ and a state $q \in \QPA$. We denote the set of runs in 
$\PiPA(q)$ induced by $c$ as $\Pi_c(q)$. Formally, $\pi = q^0 q^1 \cdots \in \Pi_c(q)$ if $q^0 = q$, and for all $i \geq 0$ 
and $i <  n_{\pi}$, $m^{i+1} = c(q^i, \sigma^i)$. Similarly, we denote the subset of runs in $\Pi_c(q)$
that eventually reach a state in $\QPA^f$ as $\Pi_c^f(q)$. Formally, $\pi \in \Pi_c^f(q)$ if there exists 
an integer $i \in \{0,\ldots,n_{\pi}\}$ such that $q^i \in \QPA^f$. For $\pi \in \Pi_c^f(q)$, we define
\[
r_{\pi} := \min \{ i \in \{0, \ldots, n_{\pi}\} ~|~ q^i \in \QPA^f \} \,.
\]

Next we define an instantaneous cost $\DPA : \EPA \rightarrow \mathbb{R}$, which satisfies $\DPA(e) > 0$ for all $e \in \EPA$, and a terminal cost $\HPA : \QPA \rightarrow \RR$. 
Now consider the run $\pi = q^0 q^1 \dots q^{n_{\pi}} \in \Pi_c^f(q)$ with $q^0 = q$, 
$c(q^i,\sigma^i) = m^{i+1}$, and $e^i := (q^i,\sigma^i,q^{i+1}) \in \EPA$.
We define a \textit{cost-to-go} $J : \QPA \times \cC \rightarrow \RR$ by
\begin{footnotesize}
\begin{eqnarray*}
J(q,c) = \begin{cases}
\max \limits_{\pi \in \Pi_c(q)}
\left \lbrace
\sum \limits_{j=0}^{r_{\pi}-1} \DPA(e^j) + \HPA(q^{r_{\pi}})
\right\rbrace, & \Pi_c(q) = \Pi^f_c (q) \\
\infty, & \textrm{~otherwise} \,.
\end{cases}
\end{eqnarray*}
\end{footnotesize}

\begin{remark}
\label{rem:costtogo}
There are several notable features of our formula for the cost-to-go.

(i) For a given $q \in \QPA$, there may be multiple runs $\pi \in \Pi_c(q)$ due to the (second, non-standard type of) non-determinism of the PA.
As such, we assume the worst case and take the maximum over $\Pi_c(q)$ in the cost-to-go.
% because of this non-determinacy in $\cAPA$
%: it is uncertain which, among the possibly multiple trajectories allowed by $c$, that will be taken so we assume the worst-case situation. 
Moreover, we require $\Pi_c(q) = \Pi^f_c (q)$ for a finite cost-to-go so that $r_{\pi}$ is well-defined and all runs starting at $q$ eventually reach $\QPA^f$.
%Moreover, we require $\Pi_c(q) = \Pi^f_c (q)$ for a finite cost-to-go, otherwise there may exist a run starting at $q$ and applying control policy $c$ that does not reach $\QPA^f$. Also, when $\Pi_c(q) = \Pi^f_c (q)$, $r_{\pi}$ is well-defined.

(ii) We have assumed that finite runs must terminate on PA states that have no outgoing edges. 
%The motivation for this choice becomes clear in light of the formulation of the cost-to-go. 
Suppose we included 
in $\Pi_c(q)$ finite prefixes of (finite or infinite) runs. These necessarily would be finite runs with final 
PA states that have outgoing edges. Then if we take a finite or infinite run that eventually reaches a goal 
PA state, certain finite prefixes of that run may not yet have reached a goal PA state, and we would get 
$\Pi_c(q) \neq \Pi_c^f(q)$ and an infinite cost-to-go. This anomaly arises from creating an artificial situation 
in which not all runs starting at an initial PA state reach a goal PA state because we included (unsuccessful) 
finite prefixes of successful runs. 

(iii) The cost-to-go function also accounts for infinite runs by using the variable $r_{\pi}$ to record the first time a goal PA state is reached and by taking the cost only over the associated prefix of the infinite run. 
Although our primary focus is on reach-avoid specifications, in which finite runs terminate on goal PA states with no outgoing edges, infinite runs allow us to extend our framework to a fragment of LTL where, for example, a goal PA state is reached always eventually; see Remark \ref{rem:reach-avoidseq} for further details. 
%Allowing infinite runs seems to contradict a reach-avoid specification in which we only want finite runs that terminate on goal PA states with no outgoing edges. 
%The reason we also allow infinite runs in the formulation of the high level plan is that it allows us to extend our framework to a fragment of LTL where, for example, a goal PA state is reached always eventually; see Remark \ref{rem:reach-avoidseq} for further details. 

\end{remark}

\begin{figure}
\centering%
\includegraphics[width=0.7\linewidth,trim=0cm 0cm 0cm 0cm, clip=true]{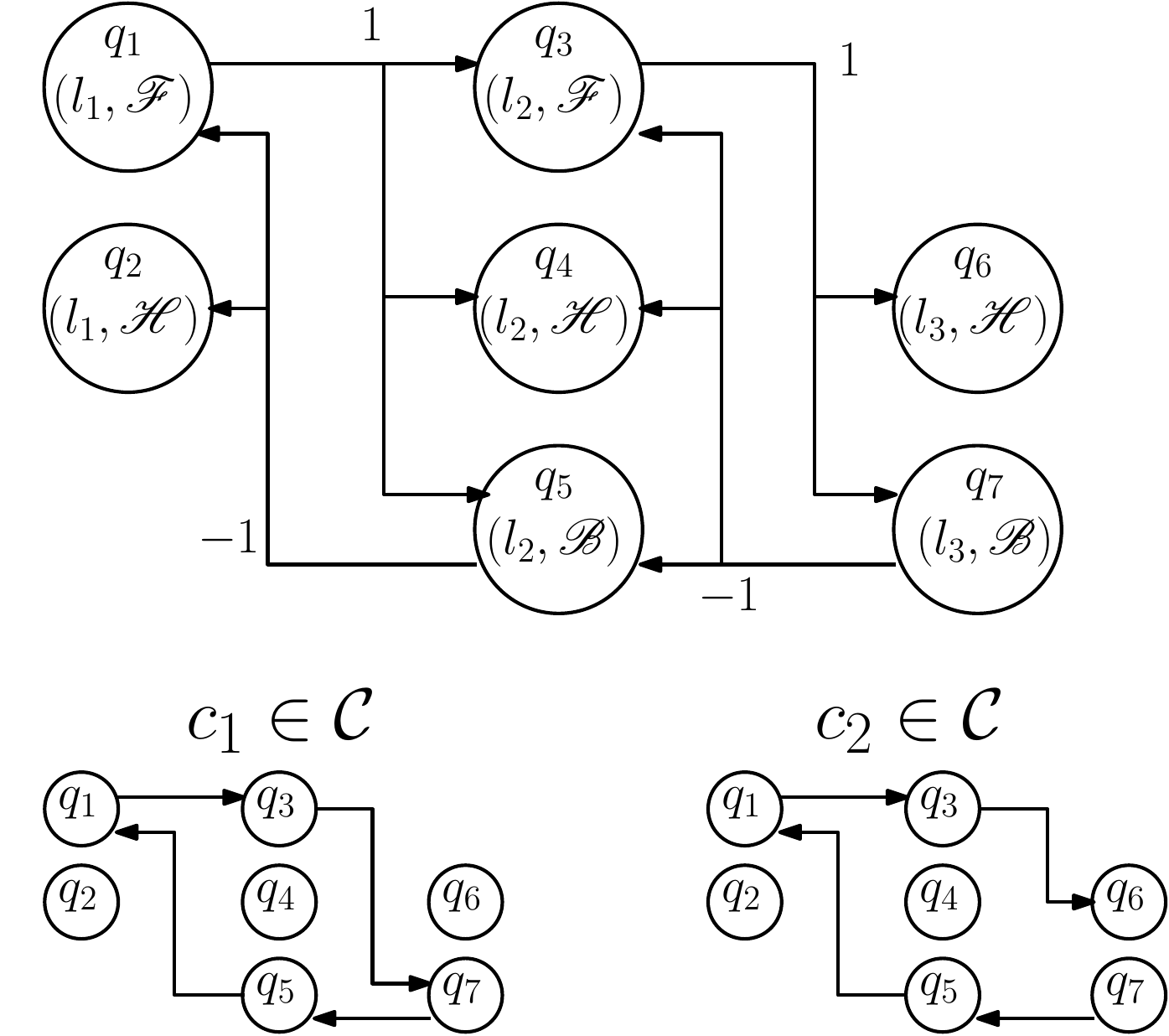}
\caption{At the top, a PA is depicted for a single output system having three motion primitives $M = \{\sH, \sF, \sB\}$ 
over three boxes $\LOTS = \{l_j \; | \; j = 1,2,3\}$. The numbers $1,-1 \in \Sigma$ on the edges (shown as arrows) 
are the corresponding labels. The bottom pictures show the reduced set of transitions induced by control polices 
$c_1, c_2 \in \cC$. 
}%
\vspace{-2mm}
\label{fig:ctrpolicyex1D}%
\end{figure}

\begin{example}
Consider the PA shown at the top of Figure~\ref{fig:ctrpolicyex1D} corresponding to a single output system with the three 
motion primitives $M = \{\sH, \sF, \sB\}$ from Example \ref{ex:MAex} over three boxes $\LOTS = \{l_j \; | \; j = 1,2,3\}$. 
%The motion primitive $\sF$ causes the output to increase  or ``move forward'', $\sB$ causes the output to ``move backward'', and $\sH$ acts as a hold. 
Suppose that $\DPA(e) = 1$ for all $e \in \EPA$ and that $\HPA = 0$ for all $q \in \QPA$. 

First consider the feasible control policy $c_1 \in \cC$ with the control assignments: $c_1(q_1,1) = \sF$, $c_1(q_3,1) = \sB$, 
$c_1(q_5,-1) = \sF$, and $c_1(q_7,-1) = \sB$. The bottom left of Figure~\ref{fig:ctrpolicyex1D} shows how the control policy 
trims away possible edges in the PA. Now suppose that $\QPA^f = \{ q_7\}$. Choosing the initial condition $q_1 \in \QPA$
and under the assumption that we do not include finite runs that terminate at PA states with outgoing edges, 
we can see that $\Pi_{c_1}(q_1)$ consists of only the single infinite run $\pi = q_1 q_3 q_7 q_5 q_1 \ldots$. 
Even though this run is infinite, $\pi \in \Pi_{c_1}^f(q_1)$, $r_{\pi} = 2$, and $J(q_1, c_1) = 2$. Similarly, 
we compute $J(q_5, c_1) = 3$, $J(q_3, c_1) = 1$, $J(q_7, c_1) = 0$, and $J(q_2, c_1) = J(q_4, c_1) = J(q_6, c_1) = \infty$. 
In contrast, the feasible control policy $c_2 \in \cC$ shown on the bottom right of Figure~\ref{fig:ctrpolicyex1D} only contains finite runs.
\tqed

%Next, consider the feasible control policy $c_2 \in \cC$ with the control assignments: $c_2(q_1,1) = \sF$, $c_2(q_3,1) = \sH$, $c_2(q_5,-1) = \sF$, and $c_2(q_7,-1) = \sB$. This control policy is shown on the bottom right of Figure~\ref{fig:ctrpolicyex1D}. Suppose that $\QPA^f = \{q_6 \}$. Then we find $J(q_7, c_2) = 4$, $J(q_5, c_2) = 3$,$J(q_1, c_2) = 2$, $J(q_3, c_2) = 1$, $J(q_6, c_2) = 0$, and $J(q_2, c_2) = J(q_4, c_2) = \infty$. The difference between $c_1$ and $c_2$ is that runs are infinite under $c_1$ but finite under $c_2$. 

%Finally, suppose we had omitted the extra condition that finite runs must terminate on PA states with no outgoing edges. Considering $c_1 \in \cC$ at $q_1 \in \QPA$, then $\Pi_{c_1}(q_1)$ would contain an infinite number of finite runs $\{q_1, q_1 q_3, q_1 q_3 q_7, \ldots \}$ as well as the infinte run already noted. In particular, the two runs $q_1, q_1 q_3 \not \in \Pi_{c_1}^f(q_1)$, so by definition of the cost-to-go, we would obtain the undesired result $J(q_1, c_1) = \infty$. A similar problem would arise with the control policy $c_2$. 
\end{example}

Next we define the \textit{value function} $V : \QPA \rightarrow \RR$ to be
\[
V(q) := \min \limits_{c \in \cC} J(q,c) \,.
\]
The value function satisfies a dynamic programming principle (DPP) that takes into account the 
non-determinacy of $\cAPA$; see \cite{BRO05} where a slightly different result is proved. The proof 
is found in the appendix.

\begin{theorem}
\label{thm:value}
Consider $q \in \QPA \setminus \QPA^f$ and suppose $| \SPA(q) | > 0$. 
Then $V$ satisfies
\begin{eqnarray}
V(q) & = & \min \limits_{c(q) \in \cM(q)}
           \left \lbrace
           \max \limits_{\sigma \in \SPA(q)}
           \{ \DPA(e) + V(q') \}
           \right \rbrace
           \label{eq:valuea} \\
     & = & \max \limits_{\sigma \in \SPA(q)}
           \left \lbrace
           \min \limits_{\bar{m} \in \cM(q,\bar{\sigma})}
           \{ \DPA(\bar{e}) + V(\bar{q}) \}
           \right \rbrace   \,,
           \label{eq:valueb}
\end{eqnarray}
where $q' = (l',c(q,\sigma)) \in \QPA$, $e = (q,\sigma,q') \in \EPA$, $\bar{q} = (\bar{l},\bar{m}) \in \QPA$, 
and $\bar{e}=(q,\sigma,\bar{q}) \in \EPA$.
\end{theorem}

Notice that for all $q \in \QPA \setminus \QPA^f$ such that $| \SPA(q) | = 0$, $V(q) = \infty$ (since there can be no 
paths to the goal). Also, for all $q \in \QPA^f$, $V(q) = \HPA(q)$. 

\begin{remark}
\label{rem:DPP} 
In \eqref{eq:valuea} of Theorem~\ref{thm:value}, it is shown that $V(q)$ can be computed using the local information of $\cM(q)$ 
instead of using all of $\cC$. In \eqref{eq:valueb}, the result is taken one step further by showing that $V(q)$ can be calculated 
using only $\cM(q,\sigma)$ for each $\sigma\in\SPA(q)$. The benefit of \eqref{eq:valueb} becomes clear when we compare the 
cardinality of the sets over which the minimizations occur. Given $q \in \QPA$, let $\SPA(q)= \{\sigma_1,\ldots,\sigma_k\}$. 
In \eqref{eq:valuea} the minimization is over $\cM(q)$, and therefore the cardinality of the minimization set is 
$\prod_{i=1}^{k} |\cM(q,\sigma_k)|$. In \eqref{eq:valueb} the minimization is over $\cM(q,\sigma)$ for each $\sigma\in\SPA(q)$, 
and therefore the cardinality of the set is $|\cM(q,\sigma)|$. While both \eqref{eq:valuea} and \eqref{eq:valueb} can be used to 
compute $V(q)$, in general \eqref{eq:valueb} will be more computationally efficient.
\end{remark}

\begin{corollary} \label{cor:policy}
Consider the control policy $c^*$ such that for all $q\in\QPA$, and $\sigma\in\SPA(q)$
\[
c^*(q,\sigma) \in
\argmin_{m'\in \cM(q,\sigma)} \{ \DPA(e) + V(q') \} \,,
\]
where $q'=(l',m')$, and $e = (q,\sigma,q')$. Then $c^*$ is an optimal control policy such that for all $q\in\QPA$, $V(q)=J(q,c^*)$.
\end{corollary}

\begin{comment}
\begin{proof}
Given $q\in\QPA$, let $c^*(q)$ be a vector of control assignments such that for all $\sigma\in\SPA(q)$
\[
c^*(q,\sigma) \in
\argmin_{m'\in \cM(q,\sigma)} \{ \DPA(e) + V(q') \} \,.
\]
Now consider $\bar{c}^* \in \cC$ such that $\bar{c}^*(q) = c^*(q)$ and $\bar{c}^*$ is any admissible optimal control policy for $q' \neq q$. In the proof of Theorem \ref{thm:value} \eqref{eq:JeqV} it is shown that
\[
J(q,\bar{c}^*) = V(q) \,.
\]
Therefore $\bar{c}^*$ is an optimal control policy at $q$ and $c^*(q)$ is an optimal vector of control assignments at $q$. Since $q$ is arbitrary, this proves the corollary.
\end{proof}
\end{comment}

Figure~\ref{fig:ctrStrat} shows a possible control policy for the scenario in Figure \ref{fig:OTSsquares}. Since there are two outputs, we use the motion primitives from Example \ref{ex:MAex} in each output; formal details are given in Section \ref{sec:paracomp}. The control policy was hand-computed. Notice that different 
routes may be taken from the same product state depending on the face reached, but ultimately the control policy ensures that all 
paths lead to the goal. 

\begin{figure}%[t]
\centering%
\includegraphics[width=0.7\linewidth,trim=0cm 0cm 0cm 0cm, clip=true]{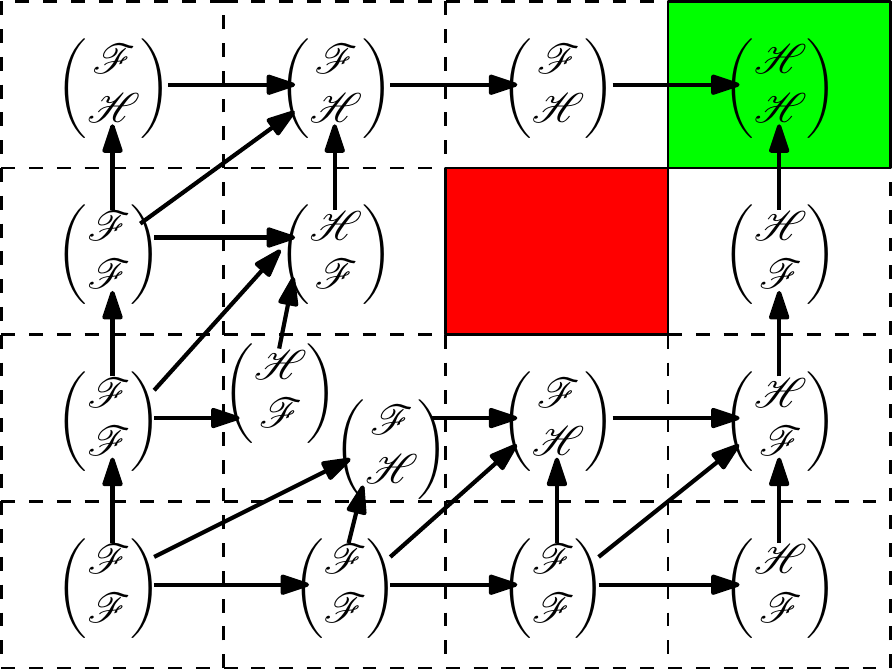}
\caption{This figure shows a discrete control strategy for the scenario shown in Figure \ref{fig:OTSsquares}. }%
\vspace{-2mm}
\label{fig:ctrStrat}%
\end{figure}

\section{Main Results}
\label{sec:main}

In this section we present our main results on a solution to Problem~\ref{prob:reachavoid}. Our final result combines 
the notion of a control policy at the high level with feedback controllers executing correct continuous time behavior 
at the low level. 
First, in accordance with the reach-avoid objective (see Remark~\ref{rem:costtogo}~(iii)), we assume the existence of motion primitives that can stabilize trajectories within a given box, that is, there exists $m \in M$ such that $\SMA(m) = \emptyset$. We restrict the final PA states to
be goal OTS states equipped with such motion primitives
\begin{equation} \label{eq:qpaf}
\QPA^f = \{ (l,m) \in \LOTS^g \times M ~|~ \SMA(m) = \emptyset \} \,.
\end{equation} 

Now suppose we have an admissible control policy $c \in \cC$ derived using Theorem~\ref{thm:value} or otherwise with
$\QPA^f$ as above. We present a complete solution to Problem~\ref{prob:reachavoid} including an initial condition 
set $\cX_0 \subset \RR^n$, a feedback control $u(x)$, and conditions on the motion primitives 
so that the reach-avoid specifications of Problem~\ref{prob:reachavoid} are met.

First we specify the initial condition set $\cX_0$. The set of feasible initial PA states is
\begin{equation} \label{eq:initPA}
\QPA^0 := \{ q \in \QPA \; | \; \Pi_c(q) = \Pi_c^f(q) \} \,.
\end{equation}
That is, a feasible initial PA state satisfies that every run (induced by the control policy) starting at the 
PA state eventually reaches a goal PA state. 

Now consider a state $x_0 \in \RR^n$. It can be used as an initial state of the system if there is some $ (l_j, m) \in \QPA^0$ for which the state is both in the box $Y_j$ and in the invariant of $m$. Recall that for all $y \in \RR^p$ and $j \in \{1,\ldots, n_L\}$, $y \in Y^*$ if and only if 
$y + d \circ l_j \in Y_j$. With this in mind, we define the set of initial states to be:
\begin{equation} 
\cX_0 = \bigcup_{(l_j,m) \in \QPA^0} 
            \bigl\{ x + h^{-1}_o(d \circ l_j)  ~|~  x \in \IMA(m)  \bigr\} \,. \label{eq:X0}
\end{equation}

Next we specify the feedback controllers to solve Problem~\ref{prob:reachavoid}. Consider any $q = (l_j,m) \in \QPA^0$.
Then for all $x \in \RR^n$ such that $x - h^{-1}_o(d \circ l_j) \in \IMA(m)$, we define the feedback 
\begin{equation}
\label{eq:u}
u(x,q) := u_m( x - h^{-1}_o( d \circ l_j ) ) \,.
\end{equation}
This defines a family of feedback controllers parametrized by $x$, the state of \eqref{eq:thesystem} and by the PA state
$q = (l_j,m)$. These feedbacks work in tandem with the control policy $c \in \cC$, which effectively determines the next 
feasible PA state $q' \in \QPA^0$. For example, suppose $q = (l_j,m) \in \QPA^0$ and suppose the label $\sigma \in \Sigma$
is measured. This event corresponds to $x \in g_e$ for $e = (m,\sigma,c(q,\sigma)) \in \EMA$. Let
$m' := c(q,\sigma)$ and let $l' \in \LOTS$ be the unique location of the OTS such that $(l,\sigma,l') \in \EOTS$. 
Then the next PA state is $q' = (l',m') \in \QPA^0$ and the controller that is applied in the next location $l' \in \LOTS$ 
is $u_{m'}(\cdot)$. 

The main result of the paper is the following. 
\begin{theorem}
\label{thm:main}
Consider the system \eqref{eq:thesystem} satisfying Assumption~\ref{assum:symmetry}, the non-empty feasible set 
$\cP \subset\RR^p$, and the goal set $\cG \subset \cP$. Let $d$ be the vector of box lengths such that the goal indices $I^g$ is non-empty. 
Consider an associated OTS $\cAOTS$, an MA $\cHMA$ satisfying Assumption~\ref{assum:MP}, a PA $\cAPA$ with $\QPA^f$ as in \eqref{eq:qpaf}, 
and an admissible control policy $c \in \cC$. Then the initial condition set $\cX_0$ given in \eqref{eq:X0} and the feedback 
controllers \eqref{eq:u} solve Problem~\ref{prob:reachavoid}.
\end{theorem}

In the remainder of this section we prove Theorem~\ref{thm:main}. 
We now give a roadmap for these results.
The verification of correctness at the low level is broken down into two steps that we now describe. 
First, we show that the MA is non-blocking in Lemma~\ref{lem:nb}. The key requirements are summarized in 
Assumption~\ref{assum:MP}. The non-blocking condition ensures that MA trajectories continually evolve in time 
and stay within the invariant regions. We also put conditions to avoid chattering in which two discrete
transitions can occur in immediate succession. While physical systems never 
undergo infinite switching in finite time, if our model predictions diverge from reality, then we have no grounds 
to claim that Problem~\ref{prob:reachavoid} is indeed solved. Second, in Lemma~\ref{lem:executionruns} we show that to each 
closed-loop trajectory of \eqref{eq:thesystem} under the feedback controllers \eqref{eq:u} and a control
policy $c \in \cC$, we can associate a unique execution of the MA (defined below) and run of the PA. 

We begin by describing the semantics of the MA. These definitions are standard; see \cite{LYG99}. 
A {\em state} of the MA is a pair $(m,x)$, where $m \in M$ and $x \in \RR^n$. Trajectories of the MA are called {\em executions} 
and are defined over hybrid time domains that identify the time intervals when the trajectory of a hybrid system is in a 
fixed motion primitive $m \in M$. Precisely, a {\em hybrid time domain} of the MA is a finite or 
infinite sequence of intervals $\tau = \{ \cI_0, \ldots, \cI_{n_\tau}\}$, such that
\begin{itemize}
\item[(i)]
$\mathcal{I}_i = [\tau_i , \tau'_i]$, for all $0 \leq i < n_{\tau}$,
\item[(ii)]
if $n_{\tau}<\infty$, then either $\mathcal{I}_{n_{\tau}} =
[\tau_{n_{\tau}}, \tau'_{n_{\tau}}]$ or $I_{n_{\tau}} =[\tau_{n_{\tau}}, \tau'_{n_{\tau}})$,
\item[(iii)]
$\tau_i \leq \tau'_i = \tau_{i+1}$, for all $0 \leq i < n_{\tau}$.
\end{itemize}

\begin{defn}
An {\em execution} of the MA is a collection $\chi = (\tau, m(\cdot), \phiMA(\cdot,x_0))$ such that
\begin{itemize}
\item[(i)]
the initial condition of the execution satisfies: $(m(0),x_0)\in \QMA^0$. 
\item[(ii)]
the continuous evolution of the execution satisfies: 
for all $i \in \{ 0, \dots, n_{\tau} \}$ with $\tau_i < \tau'_i$,
then for all $t \in [\tau_i , \tau'_i]$, $m(\cdot)$ is constant and
$\frac{d}{dt}\phiMA (t,x_0) = f( \phiMA (t,x_0), u_{m(t)} (\phiMA (t,x_0)))$, while for all $t \in [\tau_i , \tau'_i)$, $\phiMA (t,x_0)\in \IMA(m(t))$. 

\item[(iii)]
a discrete transition of the execution satisfies: 
for all $i \in \{ 0, \dots, n_{\tau}-1\}$, there exists 
$\sigma_i \in \SMA(m(\tau_i'))$ such that $(m(\tau'_i),\sigma_i,m(\tau_{i+1})) =: e_i \in \EMA$,
$\phiMA ( \tau'_i, x_0 ) \in g_{e_i}$, and $\phiMA ( \tau_{i+1}, x_0 ) = r_{e_i}( \phiMA ( \tau'_i, x_0) )$.

\end{itemize}
\end{defn}
Given an execution $\chi = (\tau,m(\cdot),\phiMA(\cdot,x_0))$, we associate to it the {\em output trajectory 
of the MA} given by $\yMA(\cdot,x_0) := h(\phiMA(\cdot,x_0))$ (the subscript $\textsc{MA}$ is included to 
avoid confusion with output trajectories $y(\cdot,x_0)$ of the physical system \eqref{eq:thesystem} which 
do not undergo resets). The {\em execution time} of an execution $\chi$ is defined as 
$\mathscr{T}(\chi) := \sum ^{n_{\tau}}_{i=0}(\tau'_i - \tau_i) = \lim_{i\rightarrow n_\tau} \tau_i' $. 
An execution is called {\em finite} if $\tau$ is a finite sequence ending with a compact time interval. 
An execution is called {\em infinite} if either $\tau$ is an infinite sequence or if 
$\mathscr{T}(\chi) = \infty$. Finally, an execution is called {\em Zeno} if it is infinite but 
$\mathscr{T}(\chi) < \infty$.

\begin{remark}
\label{rem:Zeno1}
There are two types of Zeno behavior. In one type that we call chattering, transitions are instantaneous. 
The second more subtle type is when the times between discrete transitions of the MA converge to zero, but the 
transitions are not instantaneous. Assumptions \ref{assum:MP} (i) and (iv) ensure that we cannot have chattering. 
True Zeno behavior with convergent transition times is more difficult to identify in the setting when the MA is
formed as a parallel composition. Fortunately, for our reach-avoid objective, the induced MA executions cannot be 
Zeno since there are a finite number of transitions by construction, see Lemma \ref{lem:executionruns}. 
%More detailed analysis on Zeno behavior is left as future work.
\end{remark}

\begin{defn}
The MA is {\em non-blocking} if for all $(m(0),x_0)\in \QMA^0$, the set of all infinite executions of the 
MA with initial condition $(m(0),x_0)$ is non-empty.
\end{defn}

\begin{lemma}
\label{lem:nb}
Under Assumption~\ref{assum:MP}, the MA is non-blocking.
\end{lemma}

\begin{proof}
Let $(m,x) \in \QMA^0$. If $\SMA(m) = \emptyset$, then by Assumption~\ref{assum:MP}~(vi), 
$\IMA(m)$ is invariant, so the trajectory $\phiMA(t,x)$ starting at $(m,x)$ remains in $\IMA(m)$ 
for all future time. Therefore, trivially, the MA is non-blocking for this initial condition. 
If $\SMA(m) \neq \emptyset$, then by Assumption~\ref{assum:MP}~(vii), $\phiMA(t,x)$ remains 
in $\IMA(m)$ until it reaches a guard set. Additionally, by Assumption~\ref{assum:MP} (v), the
trajectory is mapped under the reset into the next invariant. By Lemma 1 of \cite{LYG99}, the 
MA is again non-blocking for this initial condition. Overall, the MA is non-blocking.
\end{proof}

The purpose of the Assumptions~\ref{assum:MP} is to guarantee consistency between low level continuous time behavior 
and the high level discrete plan. This consistency is formalized by way of a one-to-one correspondence between infinite 
MA executions and finite PA runs, both starting from the same initial condition. The proof is found in the appendix. 

\begin{lemma} 
\label{lem:executionruns}
Suppose we have an admissible control policy $c \in \cC$, and we have an MA satisfying Assumption~\ref{assum:MP}. 
For each $(l^0,m^0) \in \QPA^0$ and $x_0 \in \IMA(m^0)$ there exist a unique infinite MA execution 
$\chi = (\tau, m(\cdot), \phiMA(\cdot, x_0))$ and a unique finite PA run $\pi = q^0 q^1 \ldots q^N$.
\end{lemma}

Before we can prove Theorem~\ref{thm:main} we need one further preliminary result stating that because of the translational 
invariance of Assumption~\ref{assum:symmetry}, the continuous part of an MA execution has a unique correspondence to a 
closed-loop trajectory of the system \eqref{eq:thesystem}. The proof is straightforward and is omitted. 

\begin{lemma}
\label{lem:traj}
Let $m \in M$, $x_0 \in \IMA(m)$, $y \in\RR^p$, and $\tilde{x}_0 = x_0 + h^{-1}_o(y)$. Consider the trajectory 
$\phi(t,\tilde{x}_0)$ of \eqref{eq:thesystem} with the feedback control $u(x) = u_{m}(x - h^{-1}_o(y))$. Also consider 
the MA trajectory $\phiMA(t,x_0)$ with feedback control $u_m(x)$. For all $t \geq 0$ such that $\phiMA(t,x_0)\in \IMA(m)$,
\begin{equation*}
\phi(t,\tilde{x}_0) = \phiMA(t,x_0) + h^{-1}_o(y).
\end{equation*}
\end{lemma}

Finally we are ready to prove Theorem~\ref{thm:main}.

\begin{proof}[Proof of Theorem~\ref{thm:main}]

We must show that (i) output trajectories of system \eqref{eq:thesystem} remain within $\cP$, and 
(ii) output trajectories eventually reach and remain within the goal set $\cG$.
Let $\tilde{x}_0 \in \cX_0$. 
Choose any $(l_{j^0}, m^0) \in \QPA^0$ such that $x_0 := \tilde{x}_0 - h^{-1}_o(d \circ l_{j^0}) \in \IMA(m^0)$.
By Lemma~\ref{lem:executionruns}, we may associate a unique MA execution $\chi$ and 
a unique PA run $\pi$ to $(l_{j^0}, m^0) \in \QPA^0$ and $x_0 \in \IMA(m^0)$. Denote the hybrid time domain as 
$\tau = \{\cI_0, \ldots, \cI_N \}$ with $\cI_k = [\tau_k, \tau_k']$ for $k = 0, \ldots, N-1$ (with $\tau_0 = 0$) and 
$\cI_N = [\tau_N, \infty)$. The last interval follows from the definition of $(l_{j^N},m^N) \in \QPA^f$ \eqref{eq:qpaf}, 
since $\SMA(m^N) = \emptyset$ and thus Assumption~\ref{assum:MP}~(vi) implies that we must have that $\cI_N = [\tau_N, \infty)$. 
As in the proof of Lemma~\ref{lem:executionruns}, denote the corresponding sequence of events as $\sigma^0 \cdots \sigma^{N-1}$.

Using Lemma~\ref{lem:traj} with $y = d \circ l_{j^0}$, we have that $\phi(t,\tilde{x}_0) = \phiMA(t,x_0) + h^{-1}_o(d \circ l_{j^0})$. 
We claim that for all $k = 0, \ldots, N$ and $t \in \cI_k$,
\begin{equation} 
\label{eq:phi2phima}
\phi(t,\tilde{x}_0) = \phiMA(t,x_0) + h^{-1}_o(d \circ l_{j^k}).
\end{equation}
Clearly the result is true for $k = 0$.

We derive two facts to assist in proving this claim. Recall that by definition of the OTS edges, we have that for all $k = 0, \ldots, N-1$, $ \sigma^{k} = l_{j^{k+1}} - l_{j^{k}}$. Furthermore, by rearranging, multiplying component-wise by $d$, and taking the preimage $h^{-1}_o$, we have the first fact: for all $k = 0, \ldots, N-1$ that $ h^{-1}_o(d \circ l_{j^{k+1}}) = h^{-1}_o(d \circ l_{j^{k}}) + h^{-1}_o(d \circ \sigma^{k})$. Also by definition of the reset map and MA execution, we get the second fact: for all $k = 0, \ldots, N-1$, $r_{e^{k}}(\phiMA(\tau_{k}', x_0)) = \phiMA(\tau_{k}', x_0) - h^{-1}_o(d \circ \sigma^{k}) = \phiMA(\tau_{k+1}, x_0)$.

Returning to \eqref{eq:phi2phima}, by induction we assume that it is true for $0 \leq k < N$ and show that it is true for $k + 1$. Using the above facts and \eqref{eq:phi2phima} for $k$ at $t = \tau_k' = \tau_{k+1}$ yields
\begin{align*}
\phi(\tau_{k+1}, \tilde{x}_0) &= \phi(\tau_k', \tilde{x}_0) = \phiMA(\tau_k',x_0) + h^{-1}_o(d \circ l_{j^k}) \\
& = (\phiMA(\tau_{k+1},x_0) + h^{-1}_o(d \circ \sigma^k)) + h^{-1}_o(d \circ l_{j^k}) \\
& = \phiMA(\tau_{k+1},x_0) + h^{-1}_o(d \circ l_{j^{k+1}}).
\end{align*}
Applying Lemma~\ref{lem:traj} with $y = h^{-1}_o(d \circ l_{j^{k+1}})$ at the new initial condition $\phiMA(\tau_{k+1},x_0) \in \IMA(m^{k+1})$, we have that for $k+1$ and for all $t \in \cI_{k+1}$ that \eqref{eq:phi2phima} holds. When $k +1 = N$, the induction terminates and the claim is proven.

Using \eqref{eq:phi2phima} and projecting to the output space we conclude that for all $k = 0, \ldots, N$ and $t \in \cI_k$, $y(t, \tilde{x}_0) \in Y_{j^k}$. Since all the boxes are contained in $\cP$ by construction, then for all $t \geq 0$ we have (i). Moreover, since $l_{j^N} \in \LOTS^g$ implies the goal box $Y_{j^N}$ is contained in $\cG$ and $\cI_N = [\tau_N, \infty)$, we have (ii).
\end{proof}

\begin{remark}
The above result does not depend on the method of construction of the admissible control policy $c \in \cC$, nor does it require the control policy to be optimal. 
This allows for different path planning techniques on the PA, as we show in Section \ref{sec:3policy}.
\end{remark}

\begin{remark} \label{rem:reach-avoidseq}
The extension to a sequence of reach-avoid problems is straightforward, following the idea in \cite{WOL13}. First, the reach property (ii) of Problem \ref{prob:reachavoid} is relaxed to $y(T,x_0) \in \cG$. Next, suppose there is a finite sequence of goals $\LOTS^{g,i}$, $i = 1, ..., n_g > 1$. In contrast to \eqref{eq:qpaf}, we set the final PA states to be $\QPA^{f,i} = \{ (l,m) \in \LOTS^{g,i} \times M ~|~ \SMA(m) \neq \emptyset \}$ for $i = 1, \ldots, n_g -1$. Finally, one must design control policies $c_i$ with associated initial conditions $\QPA^{0,i}$ \eqref{eq:initPA} such that $\QPA^{f,i} \subset \QPA^{0,i+1}$ for $i = 1, \ldots, n_g -1$. For $i = n_g$, one may impose solutions to remain invariant or connect back to the first goal. 
\end{remark}

\section{Parallel Composition of Motion Primitives}
\label{sec:paracomp}

\begin{comment}
The final observation is that motion primitives may be constructed from the composition of \textit{atomic motion primitives} 
if the dynamics for each output can be decoupled. The formal procedure of parallel composition of MA is discussed in Section \ref{sec:paracomp}. As a preliminary example, consider the case when $p=2$. Suppose that each output is generated by an independent state equation, but the two state equations are identical. Also let $Y^*=[0,d]\times[0,d]$. Suppose we have designed atomic motion primitives \textit{Backward ($\mathscr{B}$), Forward ($\mathscr{F}$), Hold ($\mathscr{H}$)} for the first output $y_1$ with the following behaviors respectively: $y_1$ leaves the interval $[0,d]$ at the face $y_1=0$, $y_1$ leaves $[0,d]$ at $y_1=d$, $y_1$ remains in $[0,d]$. We design the same atomic motion primitives for $y_2$. Informally, the parallel composition generates all the possible combinations, so that the resulting motion primitives of the composed MA are then denoted as $(\mathscr{F},\mathscr{F}), (\mathscr{B},\mathscr{F}), (\mathscr{F},\mathscr{H}), \ldots,$. 
Figure~\ref{fig:product} shows an example of this situation, assuming any location $l_j$ is associated with the box $Y^*$. 
\end{comment}

In this section we describe the operation of parallel composition of two maneuver automata. By repeated 
application of this operation, more complex higher-dimensional MA's can be constructed by starting from 
simple low dimensional atomic motion primitives, such as those described in Section~\ref{sec:MAexample}. The key 
challenge is to ensure that the resulting parallel composed MA satisfies Assumptions~\ref{assum:MP}, 
if the two constituent MA's do. This is proved in Theorem~\ref{thm:parallel}. 
First we give some preliminary definitions and we fix some notation, followed by the formal definition of 
parallel composition of MA's. 

We consider two independent systems
\begin{equation}
\dot{x}^j = f^j(x^j, u^j), \;\; y^j = h^j(x^j),
\end{equation}
where $x^j \in \RR^{n^j}$, $u^j \in \RR^{\mu^j}$, and $y^j \in \RR^{p^j}$ for $j = 1, 2$. We use superscripts 
to identify the distinct subsystems. Assume that each system satisfies Assumption~\ref{assum:symmetry}. That
is, for $j = 1, 2$, $y^j_i = x^j_i$, $i = 1,\ldots, p^j$. Associated with each system $j = 1,2$ is the MA 
\begin{equation}
\cHMA^j = (\QMA^j, \Sigma^j, \EMA^j, \XMA^j, \IMA^j, \GMA^j, \RMA^j, \QMA^{0,j}).
\end{equation}
We additionally assume that $\cHMA^1$ and $\cHMA^2$ satisfy Assumption~\ref{assum:MP}. Denote the canonical 
boxes in the respective output spaces as $Y^{*,j}= \prod_{i = 1}^{p^j} [0, d_i^j]$. 
The event sets labelling the faces of $Y^{*,j}$ are $\Sigma^j = \{ -1, 0, 1 \}^{p^j}$. 
The empty strings are denoted as $\varepsilon^j := (0, \ldots, 0) \in \Sigma^j$, $j = 1, 2$, and the
empty string is $\varepsilon := (\varepsilon^1, \varepsilon^2)$. Other sets
are similarly denoted with a superscript to identify the system, such as the set of possible events $\SMA^j(m^j)$ 
for $m^j \in M^j$ and the output indices $o^j$. For the parallel composition we also require 
some extra notation. First, for $j = 1,2$ and for each $m^j \in M^j$, define the invariant set minus all the guard sets
\begin{equation}
I^j(m^j) := \IMA^j(m^j) \setminus \left( \bigcup_{ e^j = (m^j,\sigma^j,m^j_2) \in \EMA^j} g_{e^j} \right) .
\end{equation}
Next, we need three sets: an augmented set of edges that includes a transition with the empty string, an augmented 
set of possible events for a motion primitive $m \in M^j$, and an augmented set of next feasible motion primitives. 
That is, for $j = 1, 2$, we define
\begin{eqnarray*}
\oEMA^j                  & := & \EMA^j \cup \bigl\{ (m^j, \varepsilon^j, m^j_2) ~|~ m^j, m^j_2 \in M^j \bigr. \,,  \\
                                & &       \bigl. I^j(m^j) \subset \IMA^j(m^j_2)  \bigr. \;, \\
                                & &     \bigl. (\forall e^j_2 = (m^j_2, \sigma^j_2, m^j_3) \in \EMA^j) \bigr.  \\
                                & & \bigl.  I^j(m^j) \cap g_{e^j_2} = \emptyset \bigr\} \\
\oSMA^j(m^j)               & := & \SMA^j(m) \cup \{ \varepsilon^j \} \,, \qquad \qquad m^j \in M^j \\
\overline{M}^j(m^j,\sigma^j) & := & \{ m^j_2 \in M^j ~|~ (m^j, \sigma^j, m^j_2) \in \oEMA^j \} \,, \\
                                 & & \qquad \qquad \qquad m^j \in M^j, \sigma^j \in \oSMA^j(m^j). 
\end{eqnarray*}
We also define the products of these sets: 
\begin{eqnarray*}
\oSMA(m)               & := & \oSMA^1(m^1) \times \oSMA^2(m^2) \,,     \\                 
                              & &  \qquad \qquad m = (m^1, m^2) \in M \,, \\ 
\overline{M}(m,\sigma) & := & \ol{M}^1(m^1,\sigma^1) \times \ol{M}^2(m^2, \sigma^2) \,, \\
                              & & m = (m^1, m^2) \in M, \\
                              & & \sigma = (\sigma^1, \sigma^2) \in \oSMA(m) \,.
\end{eqnarray*}
Finally, the canonical box in the output space of the parallel composition is $Y^* = Y^{*,1} \times Y^{*,2}$. 
We can now define the parallel composition of two MA's.
\begin{defn}
Consider two MA's $\cHMA^1$ and $\cHMA^2$ each satisfying Assumption~\ref{assum:MP}. The parallel 
composition  $\cHMA^1 ~||~ \cHMA^2$ is $\cHMA = ( \QMA, \Sigma, \EMA, \XMA, \IMA, \GMA, \RMA, \QMA^0 )$
where
\begin{description}

\item[State Space]
$\QMA = M \times \RR^n$ with $M = M^1 \times M^2$ and $n = n^1 + n^2$.

\item[Labels]
$\Sigma = \Sigma^1 \times \Sigma^2 = \{-1, 0, 1 \}^{p}$ with $p = p^1 + p^2$.

\item[Edges]
$\EMA \subset M \times \Sigma \times M$, where $e = (m, \sigma, m') \in \EMA$ if $\sigma \neq \varepsilon$, 
$\sigma \in \oSMA(m)$, and $m' \in \overline{M}(m, \sigma)$. Observe that for all $m \in M$, $\SMA(m) = \oSMA(m) 
\setminus \{ \varepsilon \}$.

\item[Vector Fields]
For all $m = (m^1, m^2) \in M$, 
$\XMA(m) = \begin{bmatrix} f^1(x^1, u_{m^1}(x^1)) \\ f^2(x^2, u_{m^2}(x^2)) \end{bmatrix}$.
The state is $x := (x^1, x^2) \in \RR^n$, the control input is $u := (u^1, u^2) \in \RR^{\mu}$ where 
$\mu = \mu^1 + \mu^2$, and the output is $y := (y^1, y^2) \in \RR^p$. The output map is 
$h(x) = \begin{bmatrix} h^1(x^1) \\ h^2(x^2) \end{bmatrix}$, with $o(i) = o^1(i)$ for $i = 1, \ldots, p^1$ 
and $o(i) = n^1 + o^2(i - p^1)$ for $ i = p^1 +1 , \ldots, p$.

\item[Invariants]
For all $m = (m^1, m^2) \in M$, $\IMA(m) = \IMA^1(m^1) \times \IMA^2(m^2)$.

\item[Enabling and Reset Conditions]
Consider an edge $e = (m_1, \sigma, m_2) \in \EMA$, where $m_1 = (m^1_1, m^2_1) \in M$, 
$\sigma = (\sigma^1, \sigma^2) \in \oSMA(m)$, $m_2 = (m^1_2, m^2_2) \in \overline{M}(m_1, \sigma)$, 
and $e^j = (m^j_1, \sigma^j, m^j_2) \in \oEMA^j$ for $j = 1, 2$. 
If $\sigma^j \in \oSMA^j(m^j_1)$ and $\sigma^j = \varepsilon^j$, then we define
\begin{equation*}
g_{e^j}  := I^j(m_1^j), \; \; \; \;  r_{e^j} (x^j)  :=  x^j.
\end{equation*}
Otherwise if $\sigma^j  \in \SMA^j(m^j_1)$, we have $g_{e^j} = \GMA^j(e^j)$ 
and $r_{e^j} = \RMA^j(e^j)$, corresponding to their definitions in $\cHMA^j$. Finally, we define 
$g_e = g_{e^1} \times g_{e^2}$ and $r_e(x) = \begin{bmatrix} r_{e^1}(x^1) \\ r_{e^2}(x^2) \end{bmatrix}$.

\item[Initial Conditions]
$\QMA^0 \subset \QMA$ is the set of initial conditions given by 
$\QMA^0 = \{ (m,x) ~|~ (m^j, x^j) \in \QMA^{0,j}, i = 1, 2 \}$.

\end{description}
\tqed
\end{defn}

First, notice that for each $\cHMA^j$ and for each $m^j \in M^j$, the definition of $\oEMA^j$ automatically includes self-loop 
edges $(m, \varepsilon^j, m) \in \oEMA^j$. We include such transitions with $\varepsilon^j$ so that the parallel composition 
is properly constructed. For example, suppose a proper face of $Y^{*,1}$ is crossed by the first system, but no proper face 
of $Y^{*,2}$ is crossed by the second system. To correctly account for such possibilities, the overall transition 
for the composed MA must record the lack of crossing in $Y^{*,2}$ by the empty string $\varepsilon^2$. Second, notice that we have 
allowed for additional edges with $\varepsilon^j$ to allow for the possibility of switching to a different motion primitive over 
the same box $Y^{*,j}$ if the invariants overlap and are not mapped immediately to a guard set, as can be observed by the 
definition of $\oEMA^j$. Referring to Figure \ref{fig:ctrStrat}, an edge such as $((\sF,\sH),(1,0),(\sH,\sF)) \in \EMA$ consists of $(\sF,1,\sH) \in \EMA^1$ and $(\sH,0,\sF) \in \oEMA^2$, which encodes a turn from Right to Up. 

The main result is now stated; the proof is in the appendix. 

\begin{theorem}
\label{thm:parallel}
We are given $\cHMA^1$ and $\cHMA^2$, two MA's that satisfy Assumption~\ref{assum:MP}. The parallel
composition $\cHMA = \cHMA^1 ~||~ \cHMA^2$ defined above is an MA that also satisfies Assumption~\ref{assum:MP}.
\end{theorem}

\begin{remark}
We have defined the event set as $\Sigma = \Sigma^1 \times \Sigma^2$, but the usual parallel composition of automata would have 
$\Sigma = \Sigma^1 \cup \Sigma^2$ \cite{WON15}. Given the interpretation of the event set as crossing faces of $Y^*$, the 
cartesian product is the more natural choice.

\end{remark}

\section{Motion Primitives for Integrator Systems}
\label{sec:MAexample}

In this section we give the formal details for the MA consisting of the three motion primitives Hold ($\sH$), Forward ($\sF$), and Backward ($\sB$) introduced in Example \ref{ex:MAex}. This design is able to be succinctly expressed within the MA formalism since the underlying double integrator system satisfies Assumption~\ref{assum:symmetry}. By exploiting the parallel composition construction from Section \ref{sec:paracomp}, the usefulness of this MA is demonstrated in the context of multi-robot systems in Section \ref{sec:application}. 

Suppose the nonlinear control system is the double integrator system:
\begin{equation}
\dot{x}_1 = x_2, \;\; \dot{x}_2 = u_2, \;\;\;\;\; y = x_1, \;
\end{equation}
where $x := (x_1,x_2) \in \RR^2$, $u_2 \in \RR$, and the output $y$ is the position. Each motion primitive's 
invariant region is a polytopic set in the state space defined as the convex hull of vertices $v^k_2$, 
$k \in \{1, \ldots, 6\}$; see Figure \ref{fig:DoubleVF}. The vertices are determined by the segment 
length $d > 0$, and a pre-specified maximum control value $u_2^* > 0$. Let $\bar{u}_1 := \sqrt{d u_2^*}$.  
The vertices are $v^1_2 = (0,-\bar{u}_1)$, $v^2_2 = (0,0)$, $v^3_2 = (0,\bar{u}_1)$, $v^4_2 = (d,-\bar{u}_1)$, 
$v^5_2 = (d,0)$, and $v^6_2 = (d,\bar{u}_1)$. For each motion primitive $m \in M := \{\sH, \sF, \sB\}$, we define 
an affine feedback
\begin{equation}
\label{eq:doublecontrollaw}
u_m(x) = K_m x + g_m  \,.
\end{equation}
Our specific choices are 
$K_{\sH} = \begin{bmatrix} -2u_2^*/d & -2u_2^*/\bar{u}_1 \end{bmatrix}$, 
$K_{\sF} = K_{\sB} =  \begin{bmatrix} 0 &  -2u_2^*/\bar{u}_1 \end{bmatrix}$,
$g_{\sH} = g_{\sF} = u_2^*$, and $g_{\sB}   =  -u_2^*$.
These controllers are derived using reach control theory \cite{RB06,BG14}. One first
selects control values at the vertices of the polytopes so that trajectories remain in the
invariant region (for the Hold primitive) or they exit the polytope through a certain facet and not 
through others. In particular, we have chosen all the control values at the vertices to have magnitude $u_2^*$.
%; see Figure~\ref{fig:DoubleModes}. 
Then the velocity vectors at the vertices are affinely
extended to obtain affine feedbacks over the entire polytope, yielding the vector fields shown in 
Figure~\ref{fig:DoubleVF}. 

Now we construct the MA. The state space is $\QMA = M \times \RR^2$. 
The labels are $\Sigma = \{ -1,0,1 \}$. The set of edges $\EMA$ are 
shown in Figure \ref{fig:MATrans2}. In the context of parallel composition, one may compute that the augmented edges are
\begin{equation*}
\oEMA = \EMA \cup \{ (m,0,m) \}_{m \in M} \cup \{ (\sH,0,\sF), (\sH,0,\sB) \}.
\end{equation*}

For each $m \in M$, the closed-loop vector fields are given by $[\XMA(m)](x) = (x_2, u_m(x))$, which are clearly globally Lipschitz. The invariants are given by 
the convex hull of vertices, as seen in Figure~\ref{fig:DoubleVF},
and excluding the two points $(0, 0)$ and $(d, 0)$, so the invariants are clearly bounded. For example, $\IMA(\sH) = \conv \{v_2^k\}_{k=2}^5 \setminus \{(0,0),(d,0)\}$. 
%The two excluded points allow for the the edges $(\sH,0,\sF)$ (for Assumption~\ref{assum:MP}(iv) so that guards do not map to guards). 
The enabling conditions are constructed by taking
the convex hull of vertices of the exit facet and excluding again $(0,0)$ or $(d, 0)$. Specifically, the edges $(\sF,1,\sH), (\sF,1,\sF) \in \EMA$ both have guard sets 
$g_e = \conv \{v^5_2, v^6_2 \} \setminus \{(d,0)\} = \{d\} \times (0, \bar{u}_1]$, as shown highlighted in green on the invariant region of $\sF$ in Figure \ref{fig:DoubleVF}, whereas $(\sB,-1,\sH), (\sB,-1,\sB) \in \EMA$ both have guard sets $g_e = \conv \{v^1_2, v^2_2 \} \setminus \{(0,0)\} = \{0\} \times [-\bar{u}_1,0)$. The reset conditions are constructed according to their definition. The proof of the following result is found in the appendix.

\begin{lemma}
\label{lemma:doubles}
The double integrator MA satisfies Assumption~\ref{assum:MP}.
\end{lemma}

\begin{remark}
\label{rem:Zeno2}
We noted in Remark~\ref{rem:Zeno1} that Zeno executions do not arise for reach-avoid specifications that,
by construction, involve only finite MA executions. However, one may be interested in analyzing whether an MA is
non-Zeno in its own right, independently of the high level plan or control specification for which it is used. 
It can be verified rather easily that the $p = 1$ double integrator MA design we have presented above is non-Zeno.
%It is worth noting that the two edges $(\sF,1,\sB), (\sB,-1,\sF)$ may also be added to the double integrator edges $\EMA$. They are useful in the context of a sequence of reach-avoids, although one must exercise caution as genuine Zeno behavior may result if these two edges alternate repeatedly.
The situation is considerably more complicated when considering an MA that is a parallel 
composition of these MA's or when considering an arbitrary MA. Generic conditions when hybrid systems have a 
Zeno execution have been studied in \cite{SAS99, SAS00}. However, further study of this problem
is needed in our context since existing results do not apply to all the situations that can arise in our MA.
\end{remark}

\begin{comment}
\begin{figure}[t]
\centering%
\includegraphics[width=\linewidth]{figs/DoubleModes.pdf} \\
\caption{The closed-loop vectors at the vertices for the Hold, Forward, and Backward motion primitives for 
double integrator dynamics.}%
\vspace{-2mm}
\label{fig:DoubleModes}%
\end{figure}
\end{comment}

\section{Quadrocopter Applications}
\label{sec:application}

In this section we apply our methodology to a {group of} quadrocopters. We first explain how motion primitives can be applied to the system, how to specify the reach-avoid objective, and the overall solution pipeline. Next, we compare and contrast three algorithms for computing a control policy. Then we present experimental results on three different scenarios. Lastly, we provide a discussion.

\subsection{Interfacing Multiple Quadrocopters}

%The standard quadrocopter dynamics model is ubiquitous in the literature \cite{SCHW14}. The vehicle dynamics are described by six degrees of freedom and are nonlinear. It is well known that this model is differentially flat \cite{SCHW14}. As a result, the dynamics for the position $(x_w, y_w, z_w)$ in the world frame each reduce to a double integrator and we are able to use the motion primitives from Section \ref{sec:MAexample}.

The standard quadrotor dynamical model has six degrees of freedom, which can be described by the inertial linear positions $(x_w, y_w, z_w)$ and the roll-pitch-yaw Euler angles $(\phi, \theta, \psi)$ \cite{AYAN17,MEL11}. It is well known that this system is differentially flat, relating the full state and motor inputs of the quadrotor to the flat outputs $(x_w, y_w, z_w, \psi)$ and their derivatives \cite{MEL11}. Rather than specifying positional reference trajectories, we use the motion primitives from Section \ref{sec:MAexample} independently in the $(x_w, y_w, z_w)$ directions to compute the linear accelerations as a feedback on the linear position and velocity states. Specifying an arbitrary yaw reference, differential flatness maps these linear accelerations to the $(\phi, \theta)$ angles and the total vehicle thrust, which through the use of an attitude tracking controller can be converted to motor inputs \cite{MEL11}.
%\cite{MEL11}. Finally, we use a standard attitude tracking controller to generate desired vehicle moments, which can then be converted to motor inputs \cite{MEL11}.
Although we have avoided computing motion primitives on the high dimensional nonlinear model, our experiments show that the quadrotor is fairly well approximated as double integrators in the $(x_w, y_w, z_w)$ directions using our proposed motion primitives.
% for we exploit the property that this system is differentially flat for the outputs the positions $(x_w, y_w, z_w)$ and the heading angle. 

%The vehicle dynamics are described by six degrees of freedom and are nonlinear. It is well known that this model is differentially flat \textcolor{blue}{in the robot's position $(x_w, y_w, z_w)$ and yaw angle (heading) \cite{AYAN17}. We use the motion primitives from Section \ref{sec:MAexample} to independently compute the $(x_w, y_w, z_w)$ acceleration inputs offboard as a feedback on the linear position and velocity state estimates, which are then used to compute onboard motor inputs. The yaw angle is stabilized to zero separately.}

%It is well known that \textcolor{blue}{the desired linear accelerations in the $(x_w, y_w, z_w)$ directions of the world frame can be used to compute low-level motor inputs to the quadrocopter. As such, we can use the motion primitives from Section \ref{sec:MAexample} to independently compute the $(x_w, y_w, z_w)$ acceleration inputs as a feedback on the linear position and velocity state estimates.}

We consider a centralized reach-avoid objective among $N$ quadrocopters. A copy of the gridded 3D workspace must be associated with each vehicle, resulting in a total of $p = 3N$ outputs. The $p$-dimensional MA representing the asynchronous motion capabilities of the multi-vehicle system is obtained by parallel composing $p$ times the single-output MA from Section \ref{sec:MAexample}. 

To specify the reach-avoid objective, we must identify the obstacle and goal boxes in $p = 3N$ dimensions. First we assume that the physical obstacles and goals for each vehicle are labelled on the physical 3D grid. Obstacle boxes in the output space correspond to any vehicle occupying a physical obstacle box or any two or more vehicles occupying the same physical box simultaneously.
%A margin of safety may be included to account for the quadrocopter size relative to the box size. 
To avoid the effects of downwash, we do not allow vehicles to simultaneously occupy boxes that are displaced only in the $z_w$ direction.
Goal boxes in the output space correspond to all the combinations of individual vehicle 3D goal boxes. For simplicity, we assume that each vehicle has a single 3D goal box.

The multi-vehicle reach-avoid problem is solved offline using our proposed methodology. The runtime workflow is depicted in Figure~\ref{fig:interface}. Each runtime component requires negligible computation, even for a large number of vehicles and outputs.

%The multi-vehicle reach-avoid problem can be solved using our proposed methodology, following the steps shown in Figure \ref{fig:methodology}. The output of the methodology is a hybrid controller, which consists of the motion primitive feedback controllers and a control policy, both of which are computed offline. 
%We discuss computational complexity and the modularity of our approach in Section \ref{}. 
%Once the hybrid controller is computed, the system can successfully execute the reach-avoid task from any starting configuration corresponding to a valid initial condition of the hybrid controller. The runtime workflow is depicted in Figure~\ref{fig:interface}, showing how the hybrid controller interfaces with the multi-vehicle system. 
%Due to the simplicity of a box partition and assuming that the next motion primitive can be looked-up in constant time, each runtime component requires a negligible amount of computation, even for a large number of vehicles and outputs.

\begin{figure}[t]
\centering%
\includegraphics[width=1\linewidth,trim=0cm 0cm 0cm 0cm, clip=false]{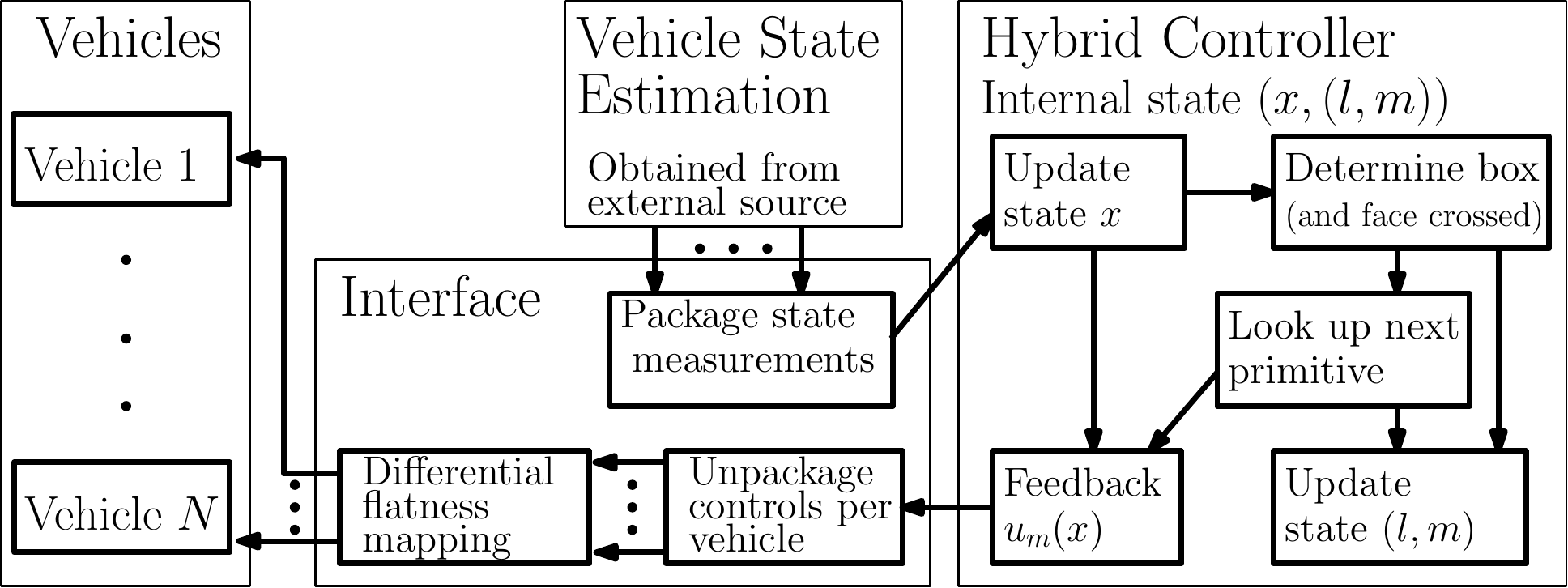}
\caption{Interface between multiple vehicles and the framework with $p = 3N$ outputs. The hybrid controller internal state consists of the joint state measurement of all the vehicles and includes the current (joint) box, $l$, and the current (joint) motion primitive, $m$. The internal state is updated via external state measurements (assumed to be given) and is used to compute the feedback controls.}%
\vspace{-2mm}
\label{fig:interface}%
\end{figure}

\subsection{Control Policy Generation} \label{sec:3policy}

We highlight three options for generating a control policy in the context of the multi-vehicle reach-avoid problem. For each, we give some implementation details and discuss its computational complexity. These are then compared in the experiments.

\subsubsection{Exhaustive Non-Deterministic Dijkstra (NDD)}

The first strategy follows the proposed methodology of Section \ref{sec:method}. We highlight our main implementation steps. First, we compute the OTS states and edges for the associated output space obstacle boxes described earlier. Second, the $p$ times parallel composed MA states and edges are computed. Third, the PA states and edges are computed. Fourth, the value function $V$ is computed using \eqref{eq:valueb}. This is done by initializing the value function to be zero at goal states and infinite elsewhere, and then propagating backwards along PA edges using a non-deterministic Dijkstra (NDD) algorithm \cite{BRO05, WOL13}. 
Once the value function is computed at all states, we compute the optimal control policy $c^{\star}$ using Corollary \ref{cor:policy}. The initial PA states \eqref{eq:initPA} correspond precisely to those states $q \in \QPA$ with $V(q) < \infty$.

The computational complexity grows exponentially as the number of inputs $p = 3N$ increases. Suppose that the physical grid has $(n_x, n_y, n_z)$ boxes in the $(x_w, y_w, z_w)$ directions. Since there are $3^p$ motion primitives, the number of PA states is bounded by $| \QPA | < (n_x n_y n_z)^N 3^p =: k_1$. 
The number of edges from an OTS state is bounded by $3^p - 1$ (the neighboring directions), whereas the number of edges from a MA state is bounded by $(2^p - 1) 3^p =: k_2$ (the neighboring directions times the possible next motion primitives). Since the MA neighboring directions are more restrictive, we have the number of PA edges is bounded by $| \EPA | < k_1 k_2$. 
The presence of obstacles can dramatically reduce the number of PA states and edges. The NDD algorithm generally must inspect all the PA states and edges to compute the value function. As a result, it is optimal and complete (with respect to the selected grid resolution and motion primitive capabilities), which results in the largest possible set of initial conditions $\cX_0$.

\subsubsection{Deterministic $\text{A}^*$}

In this strategy, we make two simplifying assumptions to compromise the quality of the control policy in exchange for better computational efficiency. First, we take the $p$ times composed MA and prune out motion primitives enabling simultaneous motion. Second, we forego computing the largest possible set of initial conditions and instead assume that a single physical initial box is specified for each vehicle. As such, it is sufficient to compute a single path of boxes in the OTS connecting the initial and goal boxes in the $p = 3N$ dimensional output space. From this path the control policy is immediately extracted, by assigning to each box the unique motion primitive leading to the next neighboring box along the path. The path is computed using a standard $\text{A}^*$ algorithm \cite{LAV06}, which starts from the initial box and propagates outwards until the goal box is reached. The (admissible) heuristic function is chosen to be the Manhattan distance, which is the sum of distances along each output direction from the current box to the goal box. 

%The computational complexity grows linearly as the number of outputs increases. 
The number of nodes that $\text{A}^*$ must investigate is bounded by the maximum number of OTS boxes, $(n_x n_y n_z)^N$, which still has exponential complexity in the number of robots. The pruned MA has $2p + 1$ motion primitives, corresponding to $\sF$ or $\sB$ in a single output component with $\sH$ elsewhere, plus the motion primitive $(\sH, \ldots, \sH)$. 
%Except for HoldAll, we also observe that each motion primitive has exactly one possible event in the active direction and that it may be followed by any other motion primitive.
Thus from the current box, we must check the $2p$ neighboring directions to select a feasible direction, taking into account out-of-bounds and obstacle configurations. 
In this implementation, the OTS, MA, and PA serve more as conceptual constructs, and do not need to be precomputed explicitly as it is expensive. In the worst case, the $\text{A}^*$ algorithm may investigate all boxes; as a result, it also produces a control policy that is complete with respect to the chosen grid and pruned MA motion capabilities. The policy produced by $\text{A}^*$ is of minimal length, but may have a long runtime execution.

\subsubsection{Deterministic Greedy Search}

This strategy also makes use of the two simplifying assumptions as with $\text{A}^*$ above, but differs in how the path is constructed. In greedy (best first) search \cite{LAV06}, the path is constructed by starting from the initial box in the output space and then extending it from the current box into any feasible neighboring direction that decreases the Manhattan distance to the goal box. 
Greedy search can often find a path very quickly, although not necessarily an optimal one. Moreover, since greedy search may fail to find a path, it is not complete.
% If a path is found, the maximum path length is $(n_x n_y n_z)N$ as a result of summing the worst case distance traveled by each robot, indicating linear complexity in the number of robots.
%When this is not possible, we extend the path in any feasible neighboring direction. 
%Since there is often more than one feasible choice, we make the choice randomly.
%The computational complexity is similar as in $\text{A}^*$ and it also finds an optimal path. In contrast, greedy search may fail to find a solution, so it is not complete.
%The main difference is that the algorithm must be specified to terminate after some maximum number of steps. As a result, it is neither complete nor optimal.

\subsection{Experimental Results}
\label{sec:experiment}

Our experimental platform is the Crazyflie 2.0; see Figure \ref{fig:exp_setup}. We used a VICON motion capture system to obtain the state estimates of the vehicles. Our implementation was done in Python 2.7.10 and ROS Kinetic, and computations were performed on a 64-bit Lenovo ThinkPad with an 8 core 3.0 GHz Intel Xeon processor and 15.4 GiB RAM. 
We illustrate three different scenarios and consider the three policy generation strategies on each of them. The corresponding video results are available at \url{http://tiny.cc/modular-3alg}.
%, and are able to convey the multi-vehicle motion much more effectively than the static plots shown here.

\subsubsection{Open Space}

\begin{figure*}
\centering%
\includegraphics[width=0.32\linewidth,trim=11cm 0cm 12cm 1.3cm, clip=true]{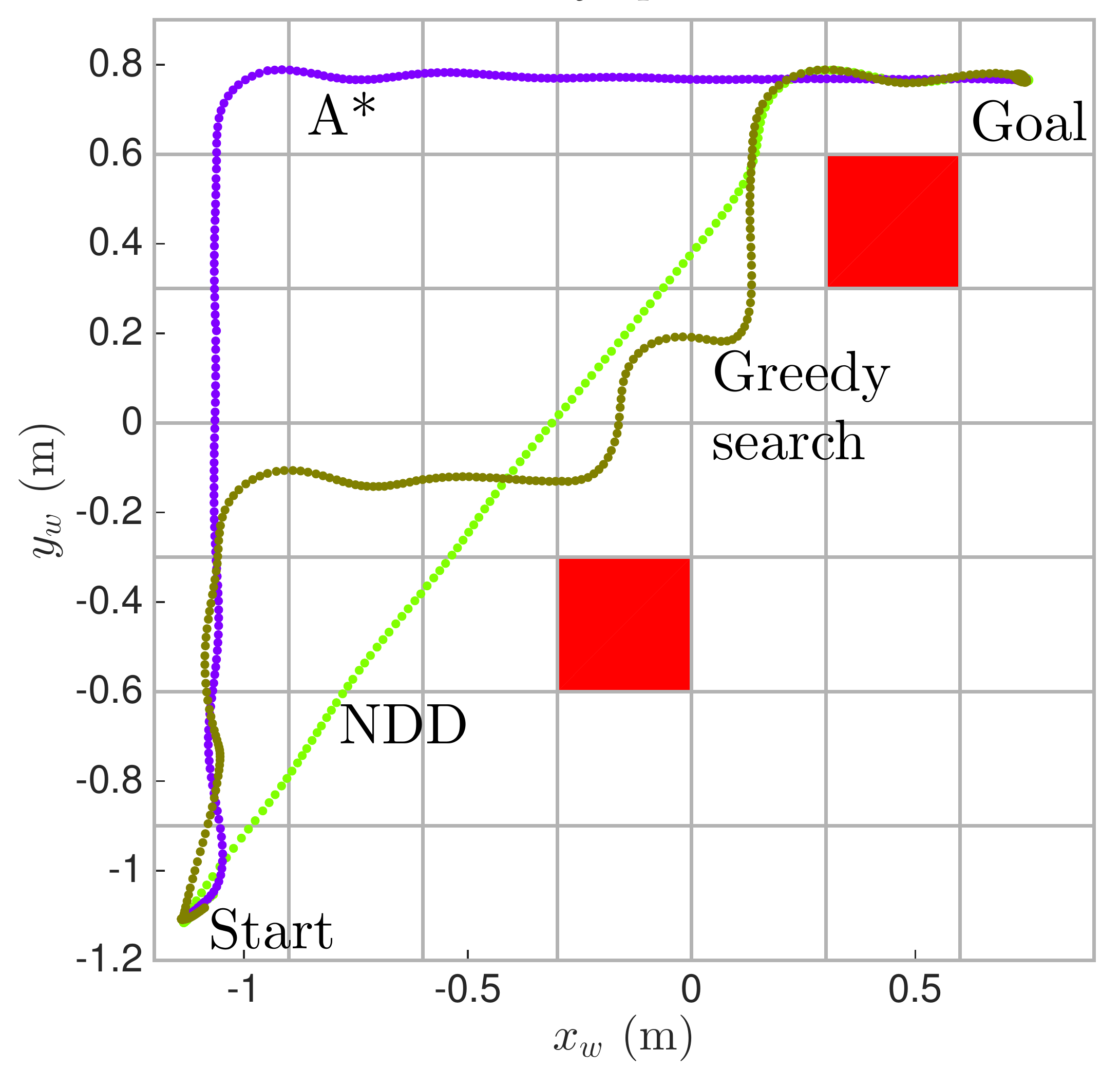}
\includegraphics[width=0.32\linewidth,trim=11cm 0cm 12cm 1.3cm, clip=true]{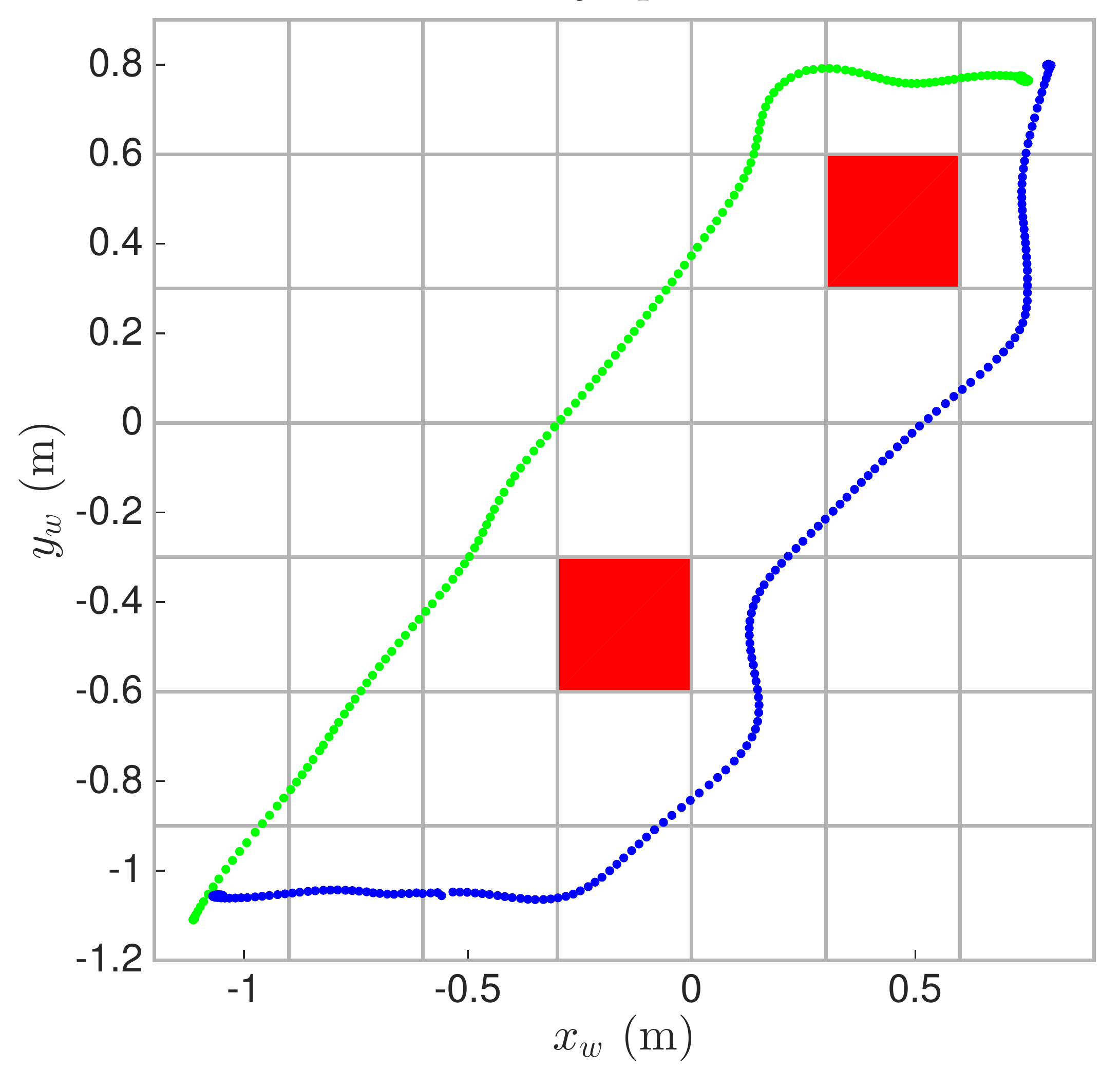}
\includegraphics[width=0.32\linewidth,trim=11cm 0cm 12cm 1.3cm, clip=true]{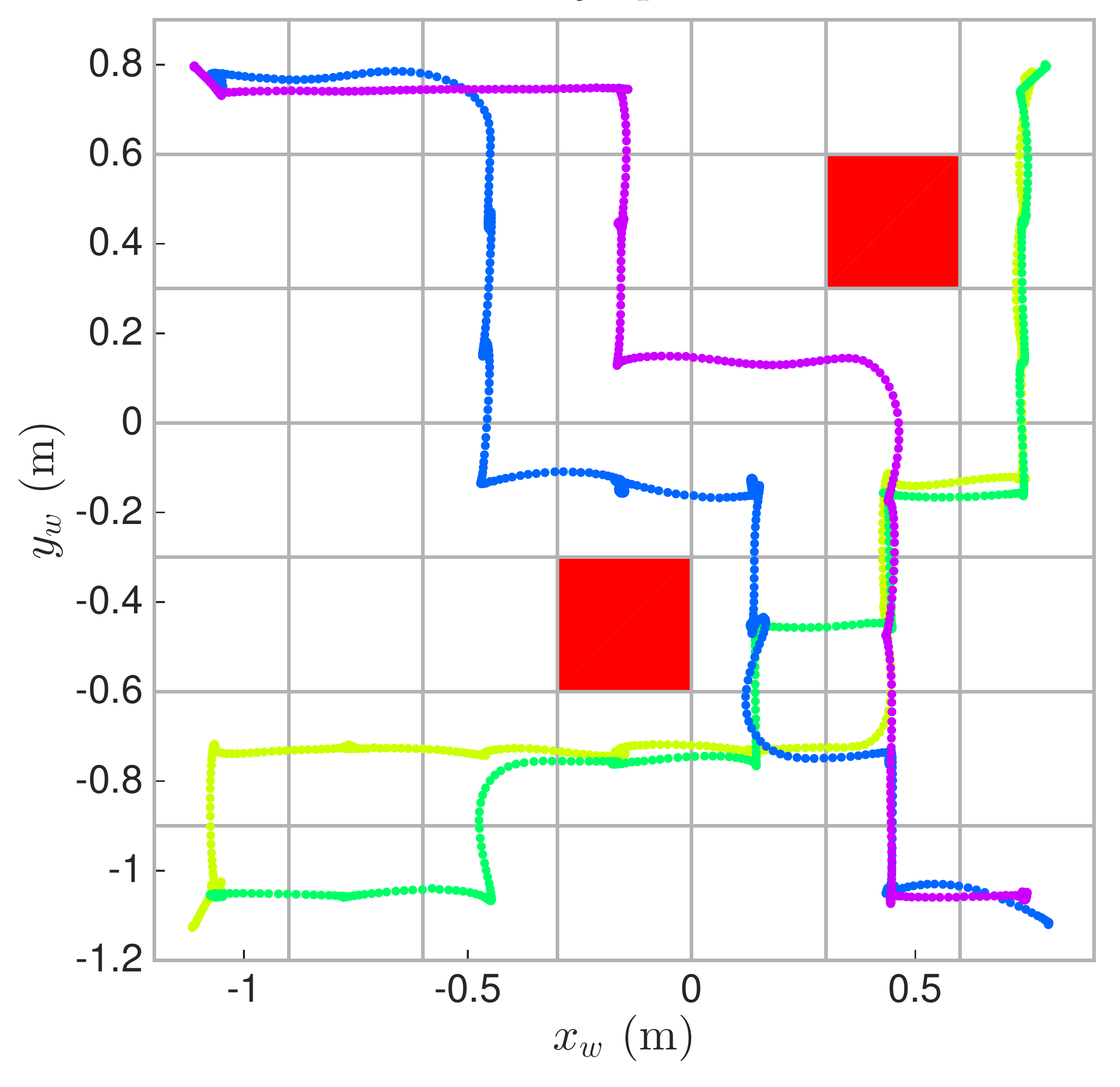}
\caption{Experimental results for the open scenario, projected onto the $(x_w, y_w)$ plane. In all plots, all the vehicles must swap corners of the room. The left plot compares trajectories for a single vehicle using the three different control policy generation strategies. The middle plot shows the resulting trajectories for two vehicles using non-deterministic Dijkstra. The right plot shows the resulting trajectories for four vehicles using greedy search, see also Figure \ref{fig:exp_setup}. Although difficult to depict, the maneuvers are safe, as the trajectories do not occupy the same physical boxes at the same time. }%
\vspace{-2mm}
\label{fig:openspace}%
\end{figure*}

The first representative scenario involves an open 3D space partitioned into a $7 \times 7 \times 2$ grid and a sparse collection of pillar-shaped obstacles. The left plot of Figure \ref{fig:openspace} compares the resulting 3D trajectories in the $(x_w, y_w)$ plane for the three strategies in the case of a single vehicle. The computation times were 40.63 milliseconds, 1.59 milliseconds, and 0.27 milliseconds for NDD, $\text{A}^*$, and greedy search, respectively. 
% NDD: OTS obs. 0.049ms, MA 6.80ms, OTS 3.67ms, PA states 2.41ms, NDD+optprims 27.70ms.
The NDD algorithm offers the best quality control policy in that there is simultaneous motion in the different degrees of freedom whenever possible and the same policy can be used from any starting box. The $\text{A}^*$ and greedy search algorithms offer similar results to each other, with both producing an optimal path of length 14. Both yield less efficient grid-like motion that is defined only along a single path from the initial box, although a new policy can quickly be recomputed from different starting boxes. Based on simulation tests for a single vehicle, each of these algorithms scale well to larger spaces or finer grids; even NDD is able to compute a solution on a $100 \times 100 \times 10$ grid in about two minutes in the worst cases. Next we compare each strategy on more vehicles.

The middle plot of Figure \ref{fig:openspace} shows the resulting trajectories for two vehicles using NDD. The control policy was computed in about 18 minutes and is defined on about PA 180000 states. 
%To accommodate the large memory requirements, the PA edges were computed locally from each PA state during execution of the NDD algorithm using the precomputed OTS and MA edges. 
%This illustrates the current capabilities of our implementation to generate a control policy in this fashion. 
While the resulting control policy yields highly efficient motion defined over a large set of initial conditions, adding more vehicles or more boxes generally explodes the computation time and memory requirements. 
Thus NDD is best suited for small scenarios involving a modest number of vehicles, when one can afford to spend time precomputing the control policy.
%In the future, more sophisticated implementations may alleviate some computational burden, such as the use of parallelization.

The right plot of Figure \ref{fig:openspace} shows the resulting trajectories for four vehicles swapping corners of the room using greedy search. Since the vehicles and physical obstacles occupy a single box, greedy search performs well, as each action typically results in one vehicle making progress towards the goal. The computation time was about four milliseconds.
%, with a control policy involving only 53 states (the length of the path). 
Simulation results on a $100 \times 100 \times 10$ grid with eight vehicles placed randomly demonstrate that greedy search is usually able to find a solution on the order of one second. As one would expect, greedy search typically fails to find a solution if long wall-like or non-convex obstacles are introduced, or if the goals are not spaced out sufficiently. Furthermore, the time to execute the entire maneuver scales with the number of vehicles.

Finally we consider the deterministic $\text{A}^*$ algorithm. Although the resulting trajectories follow a path of optimal length, they look quite similar to those found by greedy search and thus are not shown. Moreover, the method quickly becomes more computationally expensive beyond three vehicles.
% In the case of four vehicles, the computation time was X milliseconds, with a control policy involving X states. As such, significantly more time is spent searching for an optimal policy that is hardly better than ones found by greedy search. Moreover, the method quickly becomes much more computationally expensive beyond four vehicles.

\subsubsection{Channel Swapping}

The second representative scenario involves two rooms connected by a channel, defined over a $5 \times 2 \times 1$ grid, see Figure \ref{fig:channel}. Two of the vehicles must continually swap places, while the third is required to act as a gatekeeper. We specify this objective as an infinitely looping sequence of two distinct reach-avoid problems.
%, where for each reach-avoid task the goals of the first two vehicles are on opposite sides of the room and the goal of the third vehicle is at the channel. 
This illustrates that reach-avoid is a useful building block for addressing more complex specifications.

The NDD algorithm produced both control policies in about 10 seconds, while the $\text{A}^*$ algorithm took about 0.03 seconds. 
% NDD: 6.64s MA, 0.06 OTS, 0.07 + 1.44 PA, 0.21s NDD, (second policy) 0.06 OTS, 0.27+1.38 PA, 0.20 NDD
Greedy search fails to find a solution because it is unable to coordinate the third vehicle away from its goal to make space for the other two. Since the resulting trajectories overlap in physical space, Figure \ref{fig:channel_ndd} shows the trajectories as a function of time using the policy computed with NDD. The trajectories are highly non-trivial, but show that the objective is satisfied for at least one cycle of both reach-avoids. 
Although not shown, the trajectories computed using $\text{A}^*$ are similar but take a few seconds longer to execute the objective since the motion primitives are deterministic.

\begin{figure}[t]
\centering%
\includegraphics[width=0.8\linewidth,trim=0cm 0cm 0cm 0cm, clip=true]{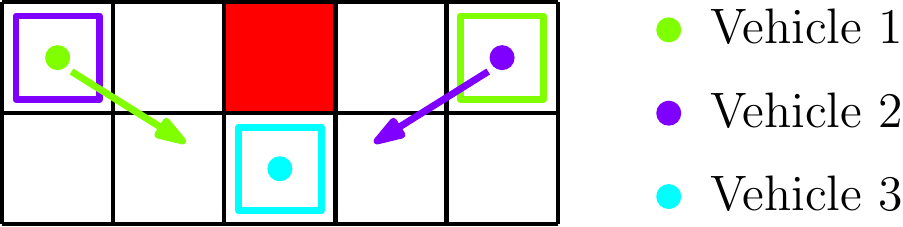}
\caption{This figure shows the channel swapping experiment involving three vehicles. In particular, it shows the specification for the first reach-avoid, where goals are shown as the colored boxes. For the second reach-avoid, the initial and goal boxes of vehicle 1 and 2 are swapped.}%
\vspace{-2mm}
\label{fig:channel}%
\end{figure}

\begin{figure}[t]
\centering%
\includegraphics[width=1\linewidth,trim=2cm 0cm 2cm 0cm, clip=true]{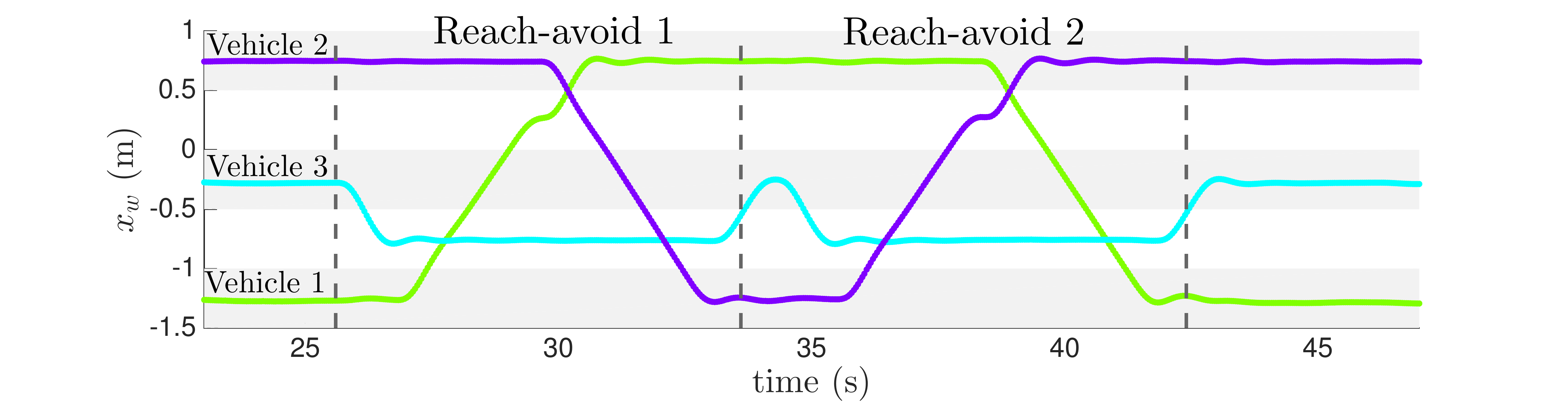}
\includegraphics[width=1\linewidth,trim=2cm 0cm 2cm 0cm, clip=true]{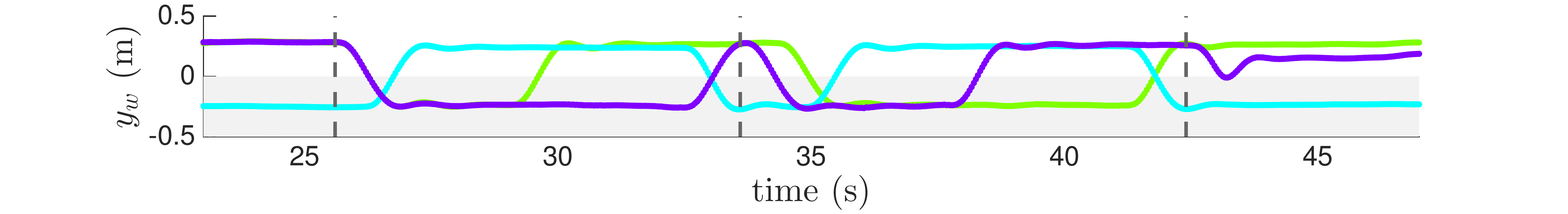}
\caption{Trajectories using non-deterministic Dijkstra for the channel experiment. The alternating grey and white areas reflect the size of the grid boxes. The duration for each of the two reach-avoid specifications is highlighted.}%
\vspace{-2mm}
\label{fig:channel_ndd}%
\end{figure}

\subsubsection{8-Puzzle}

We conclude our experimental results with the well-known 8-puzzle. On a $3 \times 3 \times 1$ grid, eight vehicles are placed randomly and must return to an ordered configuration, see Figure \ref{fig:puz8}. For this application, the $\text{A}^*$ algorithm is the most suitable, computing the control policy in 0.32 seconds. % with a path length of 24. 
The NDD approach would spend too much time precomputing edges in the high dimensional output space, while greedy search would never make progress. Results are available to view in the video.

\begin{figure}[t]
\centering%
\includegraphics[width=0.7\linewidth,trim=0cm 0cm 0cm 0cm, clip=false]{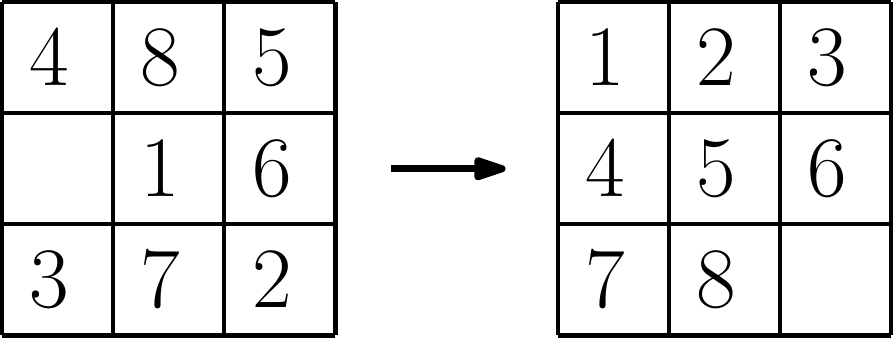}
\caption{The 8-puzzle: eight vehicles must coordinate into the ordered configuration.}%
\vspace{-2mm}
\label{fig:puz8}%
\end{figure}

\subsection{Discussion} \label{sec:discussion}

Throughout the various experimental scenarios presented, we have demonstrated the modularity offered by our approach. The designer can customize their own algorithms for generating a control policy in order to trade-off solution quality with computational efficiency. Depending on the specific application scenario, a different control policy generation strategy may be more suitable. 

In our analysis, all of the complexity was associated with the generation of the control policy for a given MA. The MA formalism enables us to generate control policies with no further regard to the continuous time trajectories that may result, due to the guarantees on discrete behavior encoded in the MA edges. On the other hand, the generation of a MA for an arbitrary system is a difficult challenge in its own right and is left to the discretion of the designer, although the design we have presented in Section \ref{sec:MAexample} can potentially be applied to control systems that are feedback-linearizable into a collection of double integrators. Taking care that the outputs are translationally invariant and that obstacle boxes can be computed, this includes end effector control of fully actuated robotic manipulators \cite{SPO05} and some wheeled vehicles through the use of look-ahead points \cite{ANDR91}.

%As long as disturbances are not too severe, the underlying feedback-based motion primitives will compensate for it. In the case of larger disturbances, the high-level control policy may be able to guide the vehicles along a new path to the goal configuration. Thus, our hybrid control strategy achieves a robust maneuver that does not require any online recomputation or generation of timed reference trajectories. These robustness features were explored more thoroughly in our previous work \cite{VUK16, VUK17}. 
Our approach offers robustness through the use of feedback-based motion primitives, as the construction of invariant regions ensures a wide range of initial conditions for which output trajectories exit through appropriate guard sets into subsequent boxes.
Since the motion primitives are updated during execution based on the measured box transitions and control policy, we do not require timing estimates for completing box transitions, which can be difficult to compute. 
%Consequently, the exact trajectory that the vehicles follow is not known in advance, but emerges dynamically as a result of the employed control policy. 
These features are advantageous under model uncertainty, which we must contend with since we base our motion primitive design on the double integrator model rather than the more complex quadrocopter model, and since aerodynamic effects arise when multiple quadrocopters fly in close proximity. Our previous work also demonstrated similar robustness of operation under wind disturbances generated by a fan on a larger quadrocopter \cite{VUK17}.
Finally, we note that our framework can easily be applied to a heterogeneous team of robots; if each vehicle has its own MA, the parallel composition automatically constructs the overall MA for the multi-vehicle system.

Of course, our solution to Problem~\ref{prob:reachavoid} is conservative because we have restricted ourselves to a particular discretization, namely the choice of a partition into boxes and the use of motion primitives. As we have demonstrated, this is a reasonable trade-off, especially since the resolution of the output space discretization, the richness of motion primitives, and the complexity of the control policy are all design parameters.

\section{Conclusion} \label{sec:conclusion}
We have developed a modular, hierarchical framework for motion planning of multiple robots in known environments.
It consists of several modules. An output transition system (OTS) models the allowable motions of the robots by partitioning their workspace into boxes. A set of motion primitives is designed based on reach control on polytopes. A maneuver automaton (MA) captures constraints on successive motion primitives. Finally, a control policy is generated based on the synchronous product of the OTS and the discrete part of the MA. 
Overall we obtain a two-level control design which is highly robust, modular, and conceptually elegant. 
We presented a specific maneuver automaton for the double integrator system, and we showed how this 
design can be composed to obtain maneuver automata for multi-robot systems. The methodology was experimentally validated on a group of quadrocopters.
Future work includes application of our methodology to different vehicle classes such as robotic manipulators or wheeled vehicles, and integration with more advanced multi-robot planning algorithms in dynamic environments.

\bibliographystyle{IEEEtranS}

\begin{comment}
\begin{IEEEbiography}
%[{\includegraphics[width=1in,height=1.25in,clip,keepaspectratio]{figs/mario.jpg}}]{Author name}
   Biography text
\end{IEEEbiography}

\begin{IEEEbiography}
[{\includegraphics[width=1in,height=1.25in,clip,keepaspectratio]{figs/zach.jpg}}]{Author name}
   Biography text
\end{IEEEbiography}

\begin{IEEEbiography}
[{\includegraphics[width=1in,height=1.25in,clip,keepaspectratio]{figs/angela2.jpg}}]{Author name}
   Biography text
\end{IEEEbiography}

\begin{IEEEbiography}
[{\includegraphics[width=1in,height=1.25in,clip,keepaspectratio]{figs/broucke_small.jpg}}]{Author name}
   Biography text
\end{IEEEbiography}
\end{comment}

\begin{appendix}

\begin{proof}[Proof of Theorem~\ref{thm:value}]
First we prove \eqref{eq:valuea}. Consider $q = (l,m) \in \QPA \setminus \QPA^f$ and suppose $|\SPA(q)|>0$. 
By the definition of $J$, for any $c \in \cC$,
\begin{equation}
\label{eq:J2a}
J(q,c) = \max \limits_{e = (q,\sigma,q') \in \EPA} \{ \DPA(e) + J(q',c) \} \,,
\end{equation}
where $q'=(l',c(q,\sigma)) \in \QPA$. Observe that given $q = (l,m) \in \QPA$ and $\sigma \in \SPA(q)$,
there exists a unique $l' \in \LOTS$ such that $(l,\sigma,l') \in \EOTS$ (since the OTS
is a deterministic automaton). Therefore, when we take the maximum over $e = (q,\sigma,q') \in \EPA$ in 
\eqref{eq:J2a} with $q' = (l',c(q,\sigma))$, the only free variable to maximize over is $\sigma \in \SPA(q)$. 
Therefore, \eqref{eq:J2a} is equivalent to
\begin{equation}
\label{eq:J2b}
J(q,c) = \max \limits_{\sigma \in \SPA(q)} \{ \DPA(e) + J(q',c) \} \,,
\end{equation}
where, as before, $q' = (l',c(q,\sigma)) \in \QPA$ and $e = (q,\sigma,q') \in \EPA$. By definition of $V$
\begin{eqnarray*}
J(q,c) \geq \max \limits_{\sigma \in \SPA(q)} \{ \DPA(e) + V(q') \} \,.
\end{eqnarray*}
Again by definition of $V$
\begin{equation*}
V(q) = \min \limits_{c \in \cC} J(q,c) \geq \min \limits_{c \in \cC}
\left\lbrace \max \limits_{\sigma \in \SPA(q)} \{ \DPA(e) + V(q') \} \right\rbrace \,.
\end{equation*}
Thus,
\begin{equation*}
V(q) \geq \min \limits_{c(q)\in \cM(q)}
\left \lbrace \max \limits_{\sigma \in \SPA(q)} \{ \DPA(e) + V(q') \} \right\rbrace \,.
\end{equation*}
To prove the reverse inequality, suppose by the way of contradiction that there exists an admissible control
assignment at $q \in \QPA$, $\hat{c}(q) \in \cM(q)$, such that
\begin{eqnarray*}
V(q)  &   >  & \max \limits_{\sigma \in \SPA(q)} \{ \DPA(\hat{e}) + V(\hat{q}) \} \\
      & \geq & \min \limits_{c(q)\in \cM(q)}
               \left\lbrace \max \limits_{\sigma \in \SPA(q)} \{ \DPA(e) + V(q') \} \right\rbrace \,,
\end{eqnarray*}
where $\hat{q} = ( \hat{l}, \hat{c}(q,\sigma) ) \in \QPA$, $\hat{e} = ( q, \sigma, \hat{q} ) \in \EPA$, 
$q' = (l',c(q,\sigma)) \in \QPA$, and $e = ( q, \sigma, q' ) \in \EPA$. Suppose the maximum for $\hat{c}(q)$ 
is achieved with $\sigma^* \in \SPA(q)$. We define $q^* = ( l^*, \hat{c}(q,\sigma^*) )$, and 
$e^* = (q,\sigma^*,q^*) \in \EPA$. Then
\[
V(q)  > \DPA(e^*) + V(q^*) \,.
\]
Suppose an admissible optimal control policy for $q^*$ to achieve $V(q^*)$ is $c^* \in \cC$. Define a new 
policy $c = c^*$ on $\QPA \setminus \{ q \}$ and $c(q) = \hat{c}(q)$. Then 
\begin{align*}
J(q,c) &= \max \limits_{\sigma \in \SPA(q)} \{ \DPA(e) + V(q') \} \\
&= \DPA(e^*) + V(q^*) < V(q) \,,
\end{align*}
a contradiction. Hence, it must be that
\[
V(q) \le \min \limits_{c(q)\in\cM(q)}
         \left\lbrace \max \limits_{\sigma \in \SPA(q)} \{ \DPA(e) + V(q') \} \right\rbrace \,,
\]
as desired. This proves \eqref{eq:valuea}.

Second we prove \eqref{eq:valueb}. Consider $q = (l,m) \in \QPA \setminus \QPA^f$ and suppose $|\SPA(q)|>0$.
Let $\bar{c}(q)\in\cM(q)$ be an admissible control assignment at $q$ such that for all $\sigma \in \SPA(q)$,
\begin{equation}
\label{eq:minc}
\bar{c}(q,\sigma) \in \argmin \limits_{\bar{m} \in \cM(q,\sigma)} \{ \DPA(\bar{e}) + V(\bar{q}) \} \,,
\end{equation}
where $\bar{q} = (\bar{l},\bar{m}) \in \QPA$, and $\bar{e}=(q,\sigma,\bar{q})\in\EPA$. We will show that
\begin{eqnarray}
\label{eq:property1}
&& \max \limits_{\sigma \in \SPA(q)} \{ \DPA(\bar{e}) + V(\bar{q}) \} \leq \\
&& \min \limits_{c(q) \in \cM(q)}
\left \lbrace
\max \limits_{\sigma \in \SPA(q)} \{ \DPA(e) + V(q') \}
\right \rbrace \,, \nonumber
\end{eqnarray}
where $\bar{q} = ( \bar{l},\bar{c}(q,\sigma) ) \in \QPA$, $\bar{e} = (q,\sigma,\bar{q}) \in \EPA$, 
$q' = ( l',c(q,\sigma) ) \in\QPA$, and $e = (q,\sigma,q') \in \EPA$. Suppose the minimum and maximum 
on the r.h.s. are achieved with $c^*(q) \in \cM(q)$ and $\sigma^* \in \SPA(q)$.
Also, suppose the maximum on the l.h.s. is achieved with $\bar{\sigma}^* \in \SPA(q)$. Then 
\eqref{eq:property1} becomes
\begin{equation}
\label{eq:property1b}
\DPA(\bar{e}^*) + V(\bar{q}^*) \le \DPA(e^*) + V(q^*) \,,
\end{equation}
where $\ol{q}^* = ( \ol{l}^*, \ol{c}(q,\ol{\sigma}^*) ) \in \QPA$, 
$\ol{e}^* = ( q, \ol{\sigma}^*,\ol{q}^* ) \in \EPA$, $q^* = ( l^*, c^*(q,\sigma^*) ) \in \QPA$, and 
$e^* = ( q, \sigma^*, q^* ) \in \EPA$. Suppose by way of contradiction that \eqref{eq:property1} does not hold.
Then \eqref{eq:property1b} does not hold. That is,
\[
\DPA(\bar{e}^*) + V(\bar{q}^*) > \DPA(e^*) + V(q^*) \,.
\]
By the maximality of $\sigma^* \in \SPA(q)$ we have
\begin{eqnarray*}
&&  \DPA(e^*) + V(q^*) \\
               &  =  & \DPA( (q, \sigma^*, (l^*,c^*(q,\sigma^*))) ) + V((l^*,c^*(q,\sigma^*))) \\
               & \ge & \DPA( (q, \ol{\sigma}^*, (\ol{l}^*, c^*(q,\ol{\sigma}^*))) ) 
                       + V( (\ol{l}^*, c^*(q,\ol{\sigma}^*)) ) \,.
\end{eqnarray*}
Therefore
\begin{eqnarray*}
&& \DPA(\bar{e}^*) + V(\bar{q}^*) \\
&=& \DPA( (q, \ol{\sigma}^*, (\ol{l}^*, \ol{c}(q,\ol{\sigma}^*))) ) + V( (\ol{l}^*, \ol{c}(q,\ol{\sigma}^* )) )  \\
&>& \DPA( (q, \ol{\sigma}^*, (\ol{l}^*,    c^*(q,\ol{\sigma}^*))) ) + V( (\ol{l}^*,    c^*(q,\ol{\sigma}^* )) ) \,.
\end{eqnarray*}
This contradicts the definition of $\bar{c}(q,\bar{\sigma}^*)$ in \eqref{eq:minc}. We conclude that
\eqref{eq:property1} must hold.

Now consider $\bar{c}^* \in \cC$ such that $\bar{c}^*(q) = \bar{c}(q)$ and $\bar{c}^*$ is any admissible 
optimal control policy for $q' \neq q$. Using \eqref{eq:J2b} we have
\begin{eqnarray*}
J( q, \bar{c}^*) &   =  & \max \limits_{\sigma \in \SPA(q)} \{ \DPA(\bar{e}) + J(\bar{q},\bar{c}^*) \} \\
                 &   =  & \max \limits_{\sigma \in \SPA(q)} \{ \DPA(\bar{e}) + V(\bar{q}) \} \,,
\end{eqnarray*}
where $\bar{q} = (\ol{l}, \bar{c}^*(q,\sigma)) = (\ol{l}, \bar{c}(q,\sigma)) \in \QPA$ and 
$\bar{e} = (q, \sigma, \bar{q}) \in \EPA$. Now by \eqref{eq:valuea} and \eqref{eq:property1} we have
\begin{eqnarray*}
J( q, \bar{c}^*) &   =  & \max \limits_{\sigma \in \SPA(q)} \{ \DPA(\bar{e}) + V(\bar{q}) \} \\
                 & \leq & \min \limits_{c(q) \in \cM(q)} 
                          \left \lbrace \max \limits_{\sigma \in \SPA(q)} \{ \DPA(e) + V(q') \} \right \rbrace
                          = V(q) \,. 
\end{eqnarray*}
By definition of $V$, we obtain that $J(q,\bar{c}^*) = V(q)$. That is,
\begin{equation}
\label{eq:JeqV}
J(q,\bar{c}^*) = V(q) = \max\limits_{\sigma \in \SPA(q)} \{ \DPA(\bar{e}) + V(\bar{q}) \}
\end{equation}
with $\bar{q} = (\ol{l}, \bar{c}(q,\sigma)) \in \QPA$ and $\bar{e} = (q, \sigma, \bar{q}) \in \EPA$.
However, we know $\bar{c}^*(q) = \bar{c}(q)$, and $\bar{c}(q)$ satisfies \eqref{eq:minc}. Therefore,
\begin{eqnarray*}
V(q) & = & \max \limits_{\sigma \in \SPA(q)}
           \left \lbrace
           \min \limits_{\ol{m} \in \cM(q,\sigma)}
           \{ \DPA(\bar{e}) + V(\bar{q}) \}
           \right \rbrace \,,
\end{eqnarray*}
where now $\bar{q} = (\ol{l}, \ol{m} ) \in \QPA$, and this proves \eqref{eq:valueb}.
\end{proof}

\begin{proof}[Proof of Lemma~\ref{lem:executionruns}]
Let $(l^0,m^0) \in \QPA^0$ and $x_0 \in \IMA(m^0)$. The initial MA state of the MA execution is 
$(m(0), x_0) = (m^0, x_0) \in \QMA^0$, and the initial PA state of the PA run $\pi$ is $q^0 = (l^0, m^0)$.
The hybrid time domain of $\chi$ will be denoted as $\tau = \{ \cI_i \}_{i = 0}^{n_{\tau}}$, and is initialized as $\tau = \{ \cI_0 \}$, where $\cI_0 = \{\tau_0 \}$ and $\tau_0 = 0$. With the base case $k = 0$ established, we construct the remainder of the MA execution and PA run by induction.

The run so far is $\pi = q^0 \ldots q^{k}$, where $q^i = (l^i, m^i)$ for $i = 0, \ldots, k$, and $\cI_k = \{\tau_k\}$.
Suppose $\SMA(m^k) = \emptyset$. Then by Assumption~\ref{assum:MP} (vi), $\phiMA(t, x_0) \in \IMA(m^k)$ for all 
$t$ in the extended interval $\cI_k = [\tau_k, \infty)$. The complete PA run is $\pi = q^0 \cdots q^k$ and the induction terminates.
Suppose instead $\SMA(m^k) \neq \emptyset$. Then by Assumption~\ref{assum:MP} (vii), there exist unique $\sigma^k \in \SMA(m^k)$ and $T^k \geq 0$, such that $\phiMA(t, x_0) \in \IMA(m^k)$ for all $t$ in the extended interval $\cI^k = [\tau_k, \tau'_k]$, $\tau'_k := \tau_k + T^k$. 
Also, for each $e = (m^k, \sigma^k, m') \in \EMA$, there exists a guard set $g_e$ such that $\phiMA(\tau'_k, x_0) \in g_e$. 
Assumption~\ref{assum:MP} (ii) tells us that for all such $m'$, the guard set is the same. Also,
Assumption~\ref{assum:MP} (iii) ensures that $\sigma^k$ is unique. 
Now we invoke the control policy to select a specific $m'$. Let $m^{k+1} := c(q^k, \sigma^k)$ so that 
$e^k := (m^k, \sigma^k, m^{k+1}) \in \EMA$ and $\phiMA(\tau'_k, x_0) \in g_{e^k}$. Define 
$x_0^{k+1} := r_{e^k}(\phiMA(\tau'_k, x_0) )$. 
By Assumption~\ref{assum:MP} (iv), $x_0^{k+1} \not \in g_e$ for any $e = (m^{k+1},\sigma,m') \in \EMA$. 
The next PA state is $q^{k+1} = (l^{k+1}, m^{k+1})$, 
where $l^{k+1} \in \LOTS$ is uniquely determined through $(l^k, \sigma^k, l^{k+1}) \in \EOTS$, by the determinism of 
the OTS. The PA run so far is $\pi = q^0 \cdots q^{k+1}$ and the new interval $\cI_{k+1} = \{\tau'_k\}$ is added to $\tau$.  

The above inductive process is guaranteed to terminate with a finite PA run by definition of $\QPA^0$. That is, since $(l^0,m^0) \in \QPA^0$ there will be a smallest $N$ such that $(l^N,m^N) \in \QPA^f$. Moreover, by definition of $\QPA^f$ \eqref{eq:qpaf}, we have that $\SMA(m^N) = \emptyset$ and so the run cannot be extended further.
The resulting MA execution is infinite with a finite number of intervals in the hybrid time domain $\tau$, 
and it is non-blocking by Lemma~\ref{lem:nb}. 
\end{proof}

\begin{proof} [Proof of Theorem~\ref{thm:parallel}]
We employ the following two standard facts regarding products, intersections, and subsets of sets.
Formally, if $A,B,C,D$ are sets, then
\begin{align} 
\label{eq:prodcart}
(A \cap C) \times (B \cap D) &= (A \times B) \cap (C \times D), \\
\label{eq:prodsub}
A \subset C \text{ and } B \subset D &\Rightarrow (A \times B) \subset (C \times D).
\end{align}

First we show that the resulting $\cHMA$ is in fact an MA according to the definition. Clearly the composed vector fields are also globally Lipschitz and the composed invariants are bounded. The non-trivial points to show are that (a) the stacked system satisfies Assumption~\ref{assum:symmetry}, (b) the invariants project within the canonical box, (c) the enabling conditions lie both within the invariant and on an appropriate face determined by $\sigma \in \Sigma$, (d) the reset conditions are determined only by the event $\sigma \in \Sigma$, and (e) the initial conditions are the entire invariants. We prove each of these in turn.

(a) We show that Assumption~\ref{assum:symmetry} for the stacked system holds. For the first condition, it can be verified by direct expansion that the definition of $h$ necessarily produces the injective output map $o: \{ 1, \ldots, p\} \rightarrow \{1, \ldots, n\}$ defined earlier. For the second condition, letting $x = (x^1, x^2)$, $u = (u^1, u^2)$ and $y = (y^1, y^2)$, we must show that $f(x,u) = f(x + h^{-1}(y), u)$. First, by Assumption~\ref{assum:symmetry} on each system, $f^j(x^j, u^j) = f^j(x^j +(h^j_{o^j})^{-1}(y^j), u^j)$. Second, it is easy (but tedious) to show that  $h^{-1}_o(y) = ((h^1_{o^1})^{-1}(y^1), (h^2_{o^2})^{-1}(y^2))$. Putting these two facts together gives the desired result.

(b) We show that for all $m \in M$, $\IMA(m) \subset h^{-1}(Y^{*})$. Letting $m = (m^1, m^2) \in M$, we have by the fact that each system is an MA that $\IMA^j(m^j) \subset (h^j)^{-1}(Y^{*,j})$ for $j = 1,2$. It is easy (but tedious) to show that $h^{-1}(Y^*) = (h^1)^{-1}(Y^{*,1}) \times (h^2)^{-1}(Y^{*,2})$. The result then follows by applying \eqref{eq:prodsub}.

(c) We show that for all $e = (m_1, \sigma, m_2) \in \EMA$, $g_e \subset  h^{-1}(\cF_{\sigma}) \cap \IMA(m_1)$. Let $e = (m_1, \sigma, m_2) \in \EMA$ and decompose it as $e^j = (m_1^j, \sigma^j, m_2^j) \in \oEMA^j$ for $j = 1,2$. For $j = 1,2$, if $\sigma^j \neq \varepsilon^j$, then $g_{e^j} \subset (h^j)^{-1}(\cF_{\sigma^j}) \cap \IMA^j(m_1^j)$ since each system is an MA. Otherwise, if $\sigma^j = \varepsilon^j$, observe that $\cF_{\varepsilon^j} = Y^{*,j}$ and $g_{e^j} = I^j(m_1^j) \subset \IMA^j(m_1^j) \subset (h^j)^{-1}(Y^{*,j})$ by construction. Consequently $g_{e^j} \subset (h^j)^{-1}(\cF_{\sigma^j}) \cap \IMA^j(m_1^j)$ again. Next, by definition $g_e = g_{e^1} \times g_{e^2}$ and $\IMA(m) = \IMA^1(m_1^1) \times \IMA^2(m_1^2)$. It is also easy (but tedious) to show that $h^{-1}(\cF_{\sigma}) = (h^1)^{-1}(\cF_{\sigma^1}) \times (h^2)^{-1}(\cF_{\sigma^2})$. The result then follows by applying \eqref{eq:prodcart} and \eqref{eq:prodsub}.

(d) We show that for all $e = (m_1, \sigma, m_2) \in \EMA$, $r_e(x) = x - h_o^{-1}(d \circ \sigma)$. Let $e = (m_1, \sigma, m_2) \in \EMA$ and decompose it as $e^j = (m_1^j, \sigma^j, m_2^j) \in \oEMA^j$ for $j = 1,2$. First, by definition $r_e(x) = (r_{e^1}(x^1), r_{e^2}(x^2))$. Then for $j = 1,2$, if $\sigma^j \neq \varepsilon^j$, then $r_{e^j}(x^j) = x^j - (h^j_{o^j})^{-1}(d^j \circ \sigma^j)$ since each system is an MA. This is also the case when $\sigma^j = \varepsilon^j$ because $r_{e^j}(x^j) = x^j$ and $(h^j_{o^j})^{-1}(d^j \circ \varepsilon^j) = 0$. Next, since $d = (d^1, d^2)$ and $\sigma = (\sigma^1, \sigma^2)$, component-wise multiplication gives $d \circ \sigma = (d^1 \circ \sigma^1, d^2 \circ \sigma^2)$. Using $x = (x^1, x^2)$ and the decomposition $h^{-1}_o(y) = ((h^1_{o^1})^{-1}(y^1), (h^2_{o^2})^{-1}(y^2))$ established in (a) with $y = d \circ \sigma$ proves the result.

(e) We must show that $\QMA^0 = \{ (m, x) \in \QMA ~|~ x \in \IMA(m) \}$. This follows immediately from the definitions of $\QMA$, $\IMA$, and $\QMA^0$.

Next we prove that (i)-(vii) of Assumption~\ref{assum:MP} hold.

(i) We must show that for all $m \in M$, $\varepsilon \not \in \SMA(m)$. This follows immediately from the definition of the edges since for all $m \in M$, $\SMA(m) = \oSMA(m) \setminus \{ \varepsilon \}$.

(ii) We must show that for all $e_1, e_2 \in \EMA$ such that $e_1 = (m_1, \sigma, m_2)$ and $e_2 = (m_1, \sigma, m_3)$, 
$g_{e_1} = g_{e_2}$. To that end, we write $e^j_1 = (m_1^j, \sigma^j, m_2^j) \in \oEMA^j$ and 
$e^j_2 = (m_1^j, \sigma^j, m_3^j) \in \oEMA^j$ for $j = 1, 2$. To show that $g_{e_1} = g_{e_2}$, we must show that $g_{e^j_1} = g_{e^j_2}$ for $j = 1, 2$. Let $j \in \{ 1, 2 \}$. If $\sigma^j = \varepsilon^j$, 
then by construction $g_{e^j_1} = I^j(m_1^j) = g_{e^j_2}$. Otherwise, 
if $\sigma^j \neq \varepsilon^j$, then $g_{e^j_1} = g_{e^j_2}$ follows from Assumption~\ref{assum:MP} (ii) on the $j$-th system.

(iii) We must show that for all $e_1, e_2 \in \EMA$ such that $e = (m_1, \sigma_1, m_2)$ and $e_2 = (m_1, \sigma_2, m_3)$, 
if $\sigma_1 \neq \sigma_2$, then $g_{e_1} \cap g_{e_2} = \emptyset$. To that end, 
we write $e^j_1 = (m_1^j, \sigma_1^j, m_2^j) \in \oEMA^j$ and $e^j_2 = (m_1^j, \sigma_2^j, m_3^j) \in \oEMA^j$
for $j = 1, 2$. If $\sigma_1 \neq \sigma_2$, then suppose w.l.o.g. that $\sigma^1_1 \neq \sigma^1_2$. 
To show $g_{e_1} \cap g_{e_2} = \emptyset$, by \eqref{eq:prodcart} it suffices to show that 
$g_{e^1_1} \cap g_{e^1_2} = \emptyset$. If both $\sigma^1_1$ and $\sigma^1_2$ are not equal to $\varepsilon^1$, 
then by Assumption~\ref{assum:MP}~(iii) $g_{e^1_1} \cap g_{e^1_2} = \emptyset$. If one of $\sigma^1_1$ or 
$\sigma^1_2$ is $\varepsilon^1$, say $\sigma_1^1$, then we cannot invoke Assumption~\ref{assum:MP}~(iii). 
However, by definition $g_{e_1^1} = I^1(m_1^1)$ is not intersecting with any other guards, so that 
$g_{e^1_1} \cap g_{e^1_2} = \emptyset$, as desired. 

(iv) We must show that for all $e_1, e_2 \in \EMA$ such that $e_1 = (m_1, \sigma_1, m_2)$ and $e_2 = (m_2, \sigma_2, m_3)$, 
$r_{e_1}(g_{e_1}) \cap g_{e_2} = \emptyset$. To that end,
we write $e^j_1 = (m_1^j, \sigma_1^j, m_2^j) \in \oEMA^j$ and $e^j_2 = (m_2^j, \sigma_2^j, m_3^j) \in \oEMA^j$
for $j = 1, 2$. By \eqref{eq:prodcart}, it suffices to show that $r_{e^j_1}(g_{e^j_1}) \cap g_{e^j_2} = \emptyset$ for at least one of 
$j = 1, 2$. Since $\varepsilon$ cannot be an event by Assumption~\ref{assum:MP}~(i), at least one of $j_1 = 1, 2$ 
must have $\sigma_1^{j_1} \neq \varepsilon^{j_1}$, and at least one of $j_2 = 1, 2$ must have 
$\sigma_2^{j_2} \neq \varepsilon^{j_2}$. Formally, there are several cases, but they can be summarized as 
follows. If there is at least one matching $j = j_1 = j_2$, where both $\sigma_1^j, \sigma_2^j \neq \varepsilon^j$, 
then both $e^j_1, e^j_2 \in \EMA^j$, and we invoke Assumption~\ref{assum:MP}~(iv) to get 
$r_{e^j_1}(g_{e^j_1}) \cap g_{e^j_2} = \emptyset$. There are two remaining cases which are similar so we only 
look at one of them. Suppose $\sigma_1^1 = \varepsilon^1$ and $\sigma_2^2 = \varepsilon^2$. Then 
$e^1_1 \not \in \EMA^1$ and $e^2_2 \not \in \EMA^2$ so we cannot invoke Assumption~\ref{assum:MP}~(iv). However, by construction $e_1^1 \in \oEMA^1$ implies that $r_{e_1^1}(g_{e_1^1}) = I^1(m_1^1)$ is not intersecting with any guards in $\EMA^1$. In particular, $e^1_2 \in \EMA^1$ since it must be that $\sigma_2^1 \neq \varepsilon^1$, and so $r_{e^1_1}(g_{e^1_1}) \cap g_{e^1_2} = \emptyset$.

(v) We must show that for all $e = (m_1, \sigma, m_2) \in \EMA$, $r_e(g_e) \subset \IMA(m_2)$. To that end, we write $e^j = (m_1^j, \sigma^j, m_2^j) \in \oEMA^j$ for $j = 1,2$. For $j = 1,2$, if $\sigma^j \neq \varepsilon^j$, then $r_{e^j}(g_{e^j}) \subset \IMA^j(m_2^j)$ follows from Assumption~\ref{assum:MP}~(v) on the individual system. This is also the case when $\sigma^j = \varepsilon^j$, since by definition of $\oEMA^j$ we have $r_{e^j}(g_{e^j}) = I^j(m_1^j) \subset \IMA^j(m_2^j)$. Next by definition we have that $\IMA(m_2) = \IMA^1(m_2^1) \times \IMA^2(m_2^2)$. Also, it is easy to verify that $r_e(g_e) = r_{e^1}(g_{e^1}) \times r_{e^2}(g_{e^2})$. The result follows by applying \eqref{eq:prodsub}.

(vi) We must show that for all $m \in M$, if $\SMA(m) = \emptyset$ then $\IMA(m)$ is invariant.
Let $m = (m^1, m^2) \in M$, $x_0 = (x_0^1, x_0^2) \in \IMA(m)$, and suppose that $\SMA(m) = \emptyset$. 
Then for $j = 1, 2$, $\SMA^j(m^j) = \emptyset$.
To see this, suppose one was not empty, say $\sigma^1 \in \SMA^1(m^1) \subset \oSMA^1(m^1)$, where by Assumption~\ref{assum:MP}~(i) $\sigma^1 \neq \varepsilon^1$. By construction, $\varepsilon^2 \in \oSMA^2(m^2)$ and by Assumption~\ref{assum:MP}~(i) proven above $(\sigma^1, \varepsilon^2) \in \oSMA(m) \setminus \{ \varepsilon \} = \SMA(m)$, so we get a contradiction. 
Appealing to the Assumption~\ref{assum:MP}~(vi) for $j = 1,2$, for all $t \geq 0$ the individual trajectories satisfy $\phiMA^j(t, x_0^j) \in \IMA^j(m^j)$. Thus for all $t \geq 0$, $\phiMA(t, x_0) \in \IMA(m)$.

(vii) We must show that for all $m \in M$, if $\SMA(m) \neq \emptyset$ then $\IMA(m)$ forces all trajectories to exit in finite time through some guard.
Let $m = (m^1, m^2) \in M$, $x_0 = (x_0^1, x_0^2) \in \IMA(m)$, and suppose that $\SMA(m) \neq \emptyset$. Reversing the argument used in Assumption~\ref{assum:MP}~(vi), now we can conclude for at least one of $j = 1,2$ that $\SMA^j(m^j) \neq \emptyset$. Suppose first that only $\SMA^1(m^1) \neq \emptyset$, then applying Assumption~\ref{assum:MP}~(vii) for $j = 1$ and Assumption~\ref{assum:MP}~(vi) for $j = 2$ furnishes the event $\sigma = (\sigma^1, \varepsilon^2) \in \SMA(m)$, with exit time $T^1$. A similar argument holds if only $\SMA^2(m^2) \neq \emptyset$. If for both $j = 1,2$, $\SMA^j(m^j) \neq \emptyset$, then Assumption~\ref{assum:MP}~(vii) for both $j = 1,2$ furnishes the event $\sigma = (\sigma^1, \sigma^2)$ with $\sigma^j\in \SMA^j(m^j)$, with exit times $T^j$. If $T^1 < T^2$, then the overall exit time is $T = T^1$ with event $\sigma = (\sigma^1, \varepsilon^2) \in \SMA(m)$. If $T^1 > T^2$, then the exit time is $T = T^2$ with event $\sigma = (\varepsilon^1, \sigma^2) \in \SMA(m)$. Otherwise, $T = T^1 = T^2$, and the event is 
$\sigma = (\sigma^1, \sigma^2) \in \SMA(m)$. In all the cases above, it is then easy (but tedious) to verify that for all $e = (m, \sigma, m') \in \EMA$ and for all $t \in [0, T]$, $\phiMA(t, x_0) \in \IMA(m)$ and $\phiMA(T, x_0) \in g_e$.

\end{proof}

\begin{proof}[Proof of Lemma~\ref{lemma:doubles}]
(i) We must show that for all $m \in M$, $0 \not \in \SMA(m)$. This is clearly true since there is no edge in $\EMA$ containing the label $0$.

(ii) We must show that for all $e_1, e_2 \in \EMA$ such that $e_1 = (m_1, \sigma, m_2)$ and $e_2 = (m_1, \sigma, m_3)$, $g_{e_1} = g_{e_2}$. This is clearly true since we have designed $g_{(\sF,1,\sH)} = g_{(\sF,1,\sF)}$ and $g_{(\sB,-1,\sH)} = g_{(\sB,-1,\sB)}$.

(iii) We must show that for all $e_1, e_2 \in \EMA$ such that $e = (m_1, \sigma_1, m_2)$ and $e_2 = (m_1, \sigma_2, m_3)$, if $\sigma_1 \neq \sigma_2$, then $g_{e_1} \cap g_{e_2} = \emptyset$. This is trivially true since for all $m \in M$, $|\SMA(m)| < 2$.

(iv) We must show that for all $e_1, e_2 \in \EMA$ such that $e_1 = (m_1, \sigma_1, m_2)$ and $e_2 = (m_2, \sigma_2, m_3)$, $r_{e_1}(g_{e_1}) \cap g_{e_2} = \emptyset$. Using Assumption \ref{assum:MP}~(ii) above, we only need to check two cases, that is, $e_1 = e_2 = (\sF,1,\sF)$ and $e_1 = e_2 = (\sB, -1,\sB)$. Both cases satisfy the condition because of the reset action on the first coordinate; for example, the first case gives $r_{e_1}(\{d\} \times (0, \bar{u}_1]) \cap \{d\} \times (0, \bar{u}_1] = \emptyset$.

(v) We must show that for all $e = (m_1, \sigma, m_2) \in \EMA$, $r_e(g_e) \subset \IMA(m_2)$. This is easily verified for all four edges in $\EMA$, for example, if $e = (\sF, 1, \sH)$, clearly $r_{e_1}(\{d\} \times (0, \bar{u}_1]) \subset \IMA(\sH)$.

(vi) We must show that for all $m \in M$, if $\SMA(m) = \emptyset$ then $\IMA(m)$ is invariant. We have that only $\SMA(\sH)=\emptyset$. As can be seen in Figure \ref{fig:DoubleVF}, the closed-loop vector field does not allow trajectories to cross outside of $\IMA(\sH)$, and therefore for all $x_0 \in \IMA(\sH)$, and for all $t\geq 0$, $\phiMA(t,x_0)\in \IMA(\sH)$.

(vii) We must show that for all $m \in M$, if $\SMA(m) \neq \emptyset$ then $\IMA(m)$ forces all trajectories to exit in finite time through some guard. Consider $\sF$ with $\SMA(\sF)= \{ 1 \}$. As can be seen in Figure \ref{fig:DoubleVF}, for all $x_0\in\IMA(\sF)$, there exists $T\geq 0$ such that for all $t\in[0,T]$, $\phiMA(t,x_0)\in \IMA(\sF)$, and $\phiMA(T,x_0)\in \{ d \} \times (0, \bar{u}_1]$. Since both $g_{(\sF,1,\sH)} = g_{(\sF,1,\sF)} = \{d\} \times (0, \bar{u}_1]$, the assumption holds. A similar argument can be made for $\sB$.
\end{proof}

\end{appendix}

\end{document}